\newcommand*{\ARXIV}{} 
\ifdefined\ARXIV
    \documentclass[numbib]{imaiai_arxiv}
\else
    \documentclass[numbib]{imaiai}
\fi
\usepackage{amsfonts}
\usepackage{amsmath}
\usepackage{amsthm}
\usepackage{cite}
\usepackage{graphics}
\usepackage{subfig}
\usepackage{url}

\usepackage{caption}
\usepackage{color}
\usepackage{algorithm}
\usepackage{algorithmic}
\usepackage{graphicx}
\usepackage{nicefrac}
\usepackage{bbm}
\usepackage{wrapfig}
\usepackage{tikz}
\usepackage[titletoc,title]{appendix}
\usepackage{stmaryrd}
\usepackage{mathtools}

\def\be{\begin{equation}}
\def\ee{\end{equation}}
\def\beas{\begin{align*}}
\def\eeas{\end{align*}}
\def\bea{\begin{align}}
\def\eea{\end{align}}

\newcommand{\h}{{\mathbf h}}
\newcommand{\x}{{\mathbf x}}
\newcommand{\y}{{\mathbf y}}

\newcommand{\uu}{{\mathbf u}}
\newcommand{\vv}{{\mathbf v}}
\newcommand{\w}{{\mathbf w}}

\newcommand{\e}{{\mathbf e}}
\newcommand{\aaa}{{\mathbf a}}
\newcommand{\bb}{{\mathbf b}}

\newcommand{\dd}{{\mathbf d}}
\newcommand{\ff}{{\mathbf f}}

\newcommand{\p}{{\mathbf p}}
\newcommand{\q}{{\mathbf q}}

\newcommand{\1}{{\mathbf 1}}

\newcommand{\A}{{\mathcal A}}
\newcommand{\B}{{\mathcal B}}

\newcommand{\X}{{\mathcal X}}

\newcommand{\R}{{\mathbb R}}
\newcommand{\N}{{\mathbb N}}

\newcommand{\abs}[1]{\left\lvert#1 \right\rvert}
\newcommand{\norm}[1]{\left\|#1 \right\|}

\newcommand{\inprod}[2]  {\left\langle{#1},{#2}\right\rangle}

\DeclareMathOperator*{\argmax}{argmax}

\newcommand{\mat}[1]{\llbracket#1\rrbracket}
\newcommand{\matflex}[1]{\left\llbracket#1\right\rrbracket}
\newcommand{\sep}[2]{\mathrm{sep}_{(#1)}\left( #2 \right)}
\newcommand{\rank}[1]{\mathrm{rank}\left( #1 \right)}
\newcommand{\state}[2]{\mathrm{states}\left( #1, #2 \right)}
\newcommand{\trajectory}[1]{\mathrm{trajectory}\left( #1 \right)}
\newcommand{\WInone}{W^{\mathrm{I}}}
\newcommand{\WHnone}{W^{\mathrm{H}}}
\newcommand{\WI}[1]{
	W^{\mathrm{I}, #1}
}
\newcommand{\WH}[1]{
	W^{\mathrm{H}, #1}
}
\newcommand{\WO}{
	W^{\mathrm{O}}
}

\def\multiset#1#2{\ensuremath{\left(\kern-.3em\left(\genfrac{}{}{0pt}{}{#1}{#2}\right)\kern-.3em\right)}}

\newcommand{\eg}{\emph{e.g.}}
\newcommand{\ie}{\emph{i.e.}}

\newcommand{\wrt}{w.r.t.}

\begin{document}
	
	\title{On the Long-Term Memory of Deep Recurrent Networks}
	
	\shorttitle{Long-Term Memory of Deep Recurrent Networks} 
	\shortauthorlist{Levine Sharir Ziv Shashua} 
	
	\author{{\sc Yoav Levine$^*$, Or Sharir, Alon Ziv and Amnon Shashua}\\[2pt]The Hebrew University of Jerusalem\\
		$^*${\email{Corresponding author: yoavlevine@cs.huji.ac.il}}}

	\maketitle
	
	\begin{abstract}
		{A key attribute that drives the unprecedented success of modern Recurrent Neural Networks (RNNs) on learning tasks which involve sequential data, is their ability to model intricate long-term temporal dependencies. 
			However, a well established measure of RNNs long-term memory capacity is lacking, and thus formal understanding of the effect of depth on their ability to correlate data throughout time is limited.
			Specifically, existing depth efficiency results on convolutional networks do not suffice in order to account for the success of deep RNNs on data of varying lengths. 
			In order to address this, we introduce a measure of the network's ability to support information flow across time, referred to as the \emph{Start-End separation rank}, which reflects the distance of the
			function realized by the recurrent network from modeling no dependency between the beginning and end of the input sequence. 
			We prove that deep recurrent networks support Start-End separation ranks which are combinatorially higher than those supported by their shallow counterparts.
			Thus, we establish that depth brings forth
			an overwhelming advantage in the ability of recurrent networks to model
			long-term dependencies, and provide an exemplar of quantifying this key attribute which may be readily extended to other RNN architectures of interest, \eg~variants of LSTM networks. We obtain our results by considering a class of
			recurrent networks referred to as Recurrent Arithmetic Circuits,
			which merge the hidden state with the input via the Multiplicative Integration
			operation, and empirically demonstrate the discussed phenomena on common RNNs.
			Finally, we employ the tool of quantum Tensor
			Networks to gain additional graphic insight regarding the complexity brought
			forth by depth in recurrent networks.}
		{Recurrent neural networks, deep learning, expressiveness, long term dependencies, tensor decompositions, tensor networks.}
	\end{abstract}
	\section{Introduction} \label{sec:intro}
	\ifdefined\SQUEEZE \vspace{-2mm} \fi
	Over the past few years, Recurrent  Neural Networks (RNNs) have become the
	prominent machine learning architectures for modeling sequential data, having
	been successfully employed for language modeling~\citep{sutskever2011generating,pascanu2013difficulty,graves2013generating}, neural
	machine translation~\citep{bahdanau2014neural}, online handwritten
	recognition~\citep{graves2009novel}, speech recognition~\citep{graves2013speech,amodei2016deep}, and more. 
	The success of recurrent networks in learning complex
	functional dependencies for sequences of varying lengths, readily implies that
	long-term and elaborate dependencies in the given inputs are somehow supported
	by these networks. Though connectivity contribution to performance of RNNs has been empirically investigated \citep{zhang2016architectural}, formal understanding of the influence of a recurrent
	network's structure on its expressiveness, and specifically on its
	ever-improving ability to integrate data throughout time (\eg~translating long
	sentences, answering elaborate questions), is lacking.
	
	An ongoing empirical effort to successfully apply recurrent networks to tasks of
	increasing complexity and temporal extent, includes augmentations of the
	recurrent unit such as Long Short Term Memory (LSTM)
	networks~\citep{hochreiter1997long} and their variants (\eg~\citep{gers2000recurrent,cho2014learning}). A parallel avenue, which we focus
	on in this paper, includes the stacking of layers to form deep recurrent
	networks~\citep{schmidhuber1992learning}.
	Deep recurrent networks, which exhibit empirical superiority over shallow
	ones (see \eg~\citep{graves2013speech}),
	implement hierarchical processing of information at every time-step that
	accompanies their inherent time-advancing computation. Evidence for a time-scale related effect arises from experiments \citep{hermans2013training}~--~deep recurrent networks appear to model dependencies which correspond to longer time-scales than shallow
	ones.
	These findings, which imply that depth brings forth a considerable advantage both in complexity and in temporal capacity of recurrent networks, have no
	adequate theoretical explanation.
	
	In this paper, we theoretically address the above presented issues. Based on the relative
	maturity of \emph{depth efficiency} results in neural networks, namely results
	that show that deep networks efficiently express functions that would require
	shallow ones to have a super-linear size (see \eg~\citep{cohen2016expressive,eldan2016power,telgarsky2015representation}), it
	is natural to assume that depth has a similar effect on the expressiveness of
	recurrent networks. Indeed, we show that depth efficiency holds for recurrent
	networks.
	
	However, the distinguishing attribute of recurrent networks,  is their inherent
	ability to cope with varying input sequence length. Thus, once
	establishing the above depth efficiency in recurrent networks, a basic question
	arises, which relates to the apparent depth enhanced long-term memory in
	recurrent networks: {\sl Do the functions which are efficiently expressed by
		deep recurrent networks correspond to dependencies over longer time-scales?}
	We answer this question affirmatively, by showing that depth provides a super-linear (combinatorial) boost to
	the ability of recurrent networks to model long-term dependencies in their inputs.
	
	In order to take-on the above question, we introduce in Section~\ref{sec:racs} a
	recurrent network referred to as a recurrent arithmetic circuit (RAC)
	that shares the architectural features of RNNs, and differs from them in
	the type of non-linearity used in the calculation. This type of connection
	between state-of-the-art machine learning algorithms and arithmetic circuits
	(also known as Sum-Product Networks~\citep{poon2011sum}) has well-established
	precedence in the context of neural networks. \citep{NIPS2011_4350} prove a
	depth efficiency result on such networks, and \citep{cohen2016expressive}
	theoretically analyze the class of Convolutional Arithmetic Circuits which
	differ from common ConvNets in the exact same fashion in which RACs differ from
	more standard RNNs. Conclusions drawn from such analyses were empirically shown
	to extend to common ConvNets (\citep{sharir2018expressive,cohen2017inductive,cohen2017boosting,levine2018deep}).
	Beyond their connection to theoretical models, RACs are similar to
	empirically successful recurrent network architectures. The modification which
	defines RACs resembles that of Multiplicative RNNs used by
	\citep{sutskever2011generating} and of Multiplicative Integration networks used
	by \citep{wu2016multiplicative}, which provide a substantial performance boost
	over many of the existing RNN models.
	In order to obtain our results, we make a connection between RACs and the Tensor Train (TT) decomposition \citep{oseledets2011tensor}, which suggests that Multiplicative RNNs may be related to a generalized TT-decomposition, similar to the way \citep{cohen2016convolutional} connected ReLU ConvNets to generalized tensor decompositions.
	
	We move on to introduce in Section~\ref{sec:corr} the notion of \emph{Start-End separation rank} as a measure of the recurrent network's ability to model elaborate long-term dependencies. In order to analyze the long-term dependencies modeled by a
	function defined over a sequential input which extends $T$ time-steps, we partition the
	inputs to those which arrive at the first $\nicefrac{T}{2}$ time-steps
	(``Start'') and the last $\nicefrac{T}{2}$ time-steps (``End''), and ask how far
	the function realized by the recurrent network is from being separable \wrt~this
	partition. Distance from separability is measured through the notion of
	separation rank~\citep{beylkin2002numerical}, which can be viewed as a
	surrogate of the $L^2$ distance from the closest separable function. For a given
	function, high Start-End separation rank implies that the function induces
	strong dependency between the beginning and end of the input sequence, and vice
	versa.
	
	In Section~\ref{sec:results} we directly address the depth enhanced long-term memory
	question above, by examining depth $L=2$ RACs and proving that functions realized by these deep networks enjoy Start-End
	separation ranks that are combinatorially higher than those of shallow networks, implying that indeed these functions can model more elaborate input dependencies
	over longer periods of time.
	An additional reinforcing result is that the
	Start-End separation rank of the deep recurrent network grows combinatorially with
	the sequence length, while that of the shallow recurrent network is
	\emph{independent} of the sequence length. Informally, this implies that vanilla
	shallow recurrent networks are inadequate in modeling dependencies of long input
	sequences, since in contrast to the case of deep recurrent networks, the modeled
	dependencies achievable by shallow ones do not adapt to the actual length of
	the input. Finally, we present and motivate a quantitative conjecture by which the Start-End
	separation rank of recurrent networks grows combinatorially with the network depth. A proof of this conjecture, which provides an even deeper insight regarding the advantages of depth in recurrent networks, is left as an open problem.
	
	Finally, in Section~\ref{sec:exp} we present numerical evaluations which support of the above theoretical findings. Specifically, we perform two experiments that directly test the ability of recurrent networks to model complex long-term temporal dependencies. 
    Our results exhibit a clear boost in memory capacity of deeper recurrent networks relative to shallower networks that are given the same amount of resources, and thus directly demonstrate the theoretical trends established in this paper.
	
	\ifdefined\SQUEEZE \vspace{-4mm} \fi
	
	\section{Recurrent Arithmetic Circuits} \label{sec:racs}
	\ifdefined\SQUEEZE \vspace{-2mm} \fi
	
	In this section, we introduce a class of recurrent networks referred to as
	Recurrent Arithmetic Circuits (RACs), which shares the architectural features of
	standard RNNs. As demonstrated below, the operation of RACs on sequential data
	is identical to the operation of RNNs, where a hidden state mixes information
	from previous time-steps with new incoming data (see
	Figure~\ref{fig:recurrent_net}). The two classes differ only in the type of
	non-linearity used in the calculation, as described by
	Equations~\eqref{eq:shallow_rn}-\eqref{eq:g_rac}. In the following sections, we utilize
	the algebraic properties of RACs for proving results regarding their ability to
	model long-term dependencies of their inputs.
	
	
	\begin{figure}
		\centering
		\includegraphics[width=\linewidth]{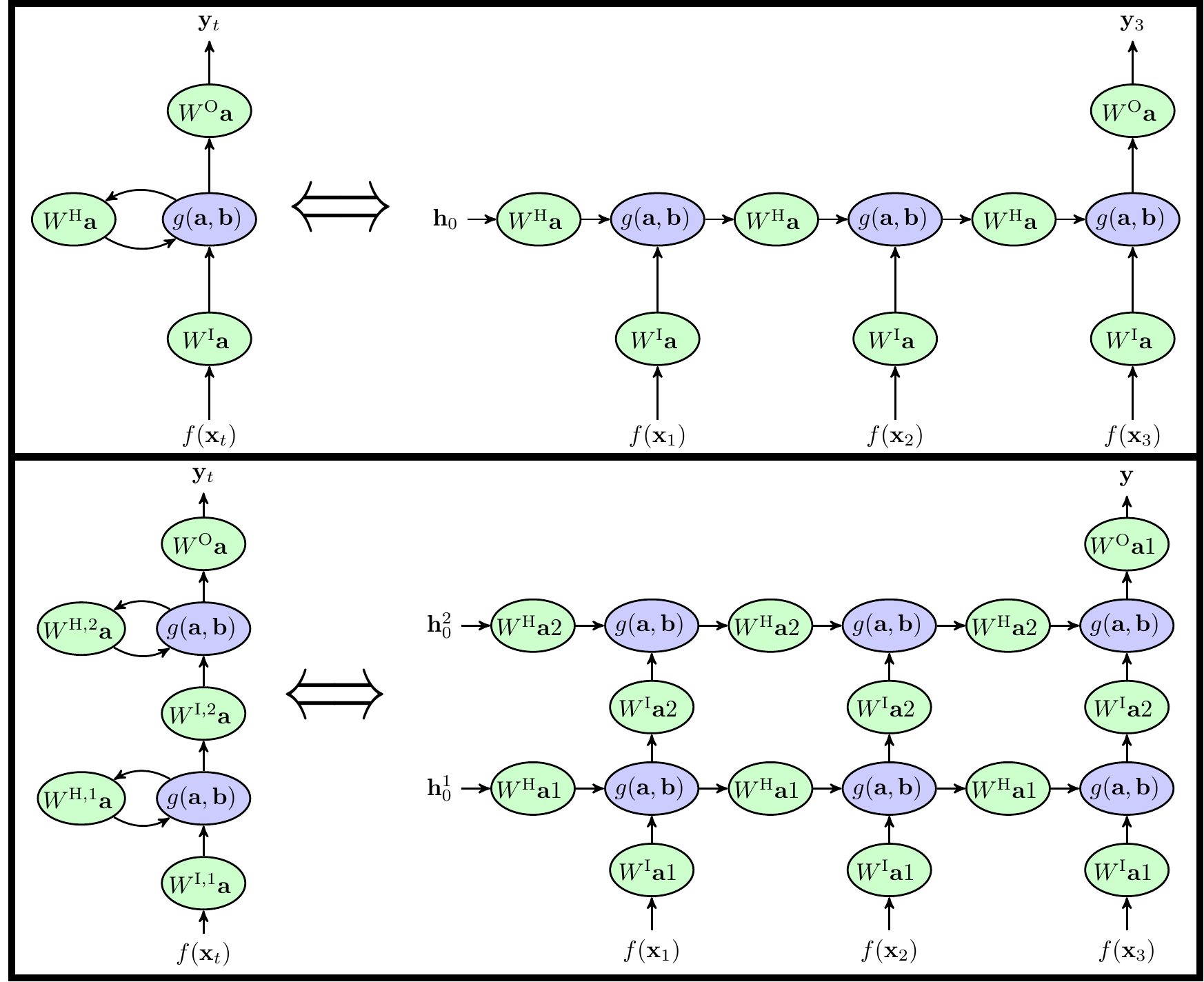}
		\ifdefined\SQUEEZE \vspace{-6mm} \fi
		\caption{Shallow and deep recurrent networks, as described by Equations~\eqref{eq:shallow_rn} and~\eqref{eq:deep_rn}, respectively.}\ifdefined\SQUEEZE \vspace{-6mm} \fi
		
		\label{fig:recurrent_net}
	\end{figure}
	
	We present below the basic framework of shallow recurrent networks
	(top of Figure~\ref{fig:recurrent_net}), which
	describes both the common RNNs and the newly introduced RACs. A recurrent
	network is a network that models a discrete-time dynamical system; we focus on
	an example of a sequence to sequence classification task into one of the
	categories $\{1,...,C\}\equiv[C]$. Denoting the temporal dependence by $t$,
	the sequential input to the network is $\{\x^t\in \X\}_{t=1}^T$, and the output is a sequence of class scores
	vectors $\{\y^{t,L,\Theta}\in\R^C\}_{t=1}^T$, where $L$ is the network depth, $\Theta$
	denotes the parameters of the recurrent network, and $T$ represents the extent of the sequence in time-steps.
	We assume the input lies in some input space $\X$ that may be discrete (e.g.
	text data) or continuous (e.g. audio data), and that some initial mapping
	$\ff:\X\rightarrow\R^M$ is preformed on the input, so that all input types are
	mapped to vectors $\ff(\x^t)\in \R^M$. The function $\ff(\cdot)$ may be viewed
	as an encoding, e.g. words to vectors or images to a final dense layer
	via some trained ConvNet. The output at time $t\in[T]$ of the shallow (depth
	$L=1$) recurrent network with $R$ hidden channels, depicted at
	the top of Figure~\ref{fig:recurrent_net}, is given by:
	\begin{align}\label{eq:shallow_rn}
	\h^{t} &= g\left(\WHnone \h^{t-1},\WInone \ff(\x^t)\right) \\
	\y^{t,1, \Theta} &= \WO \h^{t}, \nonumber
	\end{align}
	where $\h^t\in\R^R$ is the hidden state of the network at time $t$  ($\h^0$ is some initial hidden state), $\Theta$
	denotes the learned parameters $\WInone \in \R^{R \times M},
	\WHnone \in \R^{R \times R}, \WO \in \R^{C \times R}$, which are the input, hidden
	and output weights matrices respectively, and $g$ is some non-linear operation.
	We omit a bias term for simplicity.
	For common RNNs, the non-linearity is given by:
	\begin{equation}\label{eq:g_rnn}
	g^{\mathrm {RNN}}(\aaa,\bb) = \sigma(\aaa+\bb),
	\end{equation}
	where $\sigma(\cdot)$ is typically some point-wise non-linearity such as
	sigmoid, tanh etc. For the newly introduced class of RACs, $g$ is given by:
	\begin{equation} \label{eq:g_rac}
	g^{\mathrm {RAC}}(\aaa,\bb) = \aaa \odot \bb,
	\end{equation}
	where the operation $\odot$ stands for element-wise multiplication between
	vectors, for which the resultant vector upholds
	$(\aaa \odot \bb)_i = a_i \cdot b_i$. This form of merging the input and the
	hidden state by multiplication rather than addition is referred to as
	Multiplicative Integration~\citep{wu2016multiplicative}.
	
	The extension to deep recurrent networks is natural, and we follow the common
	approach (see e.g. \citep{hermans2013training}) where each layer acts as a
	recurrent network which receives the hidden state of the previous layer as its
	input. The output at time $t$ of the depth $L$ recurrent network with $R$
	hidden channels in each layer, depicted at the bottom of
	Figure~\ref{fig:recurrent_net}, is constructed
	by the following:
	\begin{align}
	\h^{t,l} &= g\left(\WH{l} \h^{t-1,l}, \WI{l} \h^{t,l-1}\right) \nonumber\\
	\h^{t,0} &\equiv \ff(\x^t) \label{eq:deep_rn} \\
	\y^{t,L, \Theta} &= \WO \h^{t,L}, \nonumber
	\end{align}
	where $\h^{t,l}\in\R^R$ is the state of the depth $l$ hidden unit at time $t$ ($\h^{0,l}$ is some initial hidden state per layer),
	and $\Theta$ denotes the learned parameters. Specifically,
	$\WI{l} \in \R^{R \times R} ~(l>1), \WH{l}\in\R^{R\times R}$
	are the input and hidden weights matrices at depth $l$, respectively. For $l=1$,
	the weights matrix which multiplies the inputs vector has the appropriate
	dimensions: $\WI{1}\in \R^{R \times M}$. The output weights matrix is
	$\WO\in \R^{C \times R}$ as in the shallow case, representing a final
	calculation of the scores for all classes $1$ through $C$ at every time-step.
	The non-linear operation $g$ determines the type of the deep recurrent network,
	where a common deep RNN is obtained by choosing $g=g^{\mathrm {RNN}}$
	[Equation~\eqref{eq:g_rnn}], and a deep RAC is obtained for
	$g=g^{\mathrm {RAC}}$~[Equation~\eqref{eq:g_rac}].
	
	We consider the newly presented class of RACs to be a good surrogate of common
	RNNs. Firstly, there is an obvious structural resemblance between the two
	classes, as the recurrent aspect of the calculation has the exact same form in
	both networks (Figure~\ref{fig:recurrent_net}). In fact, recurrent networks that
	include Multiplicative Integration similarly to RACs (and include additional non-linearities), have been shown to
	outperform many of the existing RNN models~\citep{sutskever2011generating,
		wu2016multiplicative}. Secondly, as mentioned above, arithmetic circuits have
	been successfully used as surrogates of convolutional networks. The fact that
	\citep{cohen2016convolutional} laid the foundation for extending the proof
	methodologies of convolutional arithmetic circuits to common ConvNets with ReLU
	activations, suggests that such adaptations may be made in the recurrent
	network analog, rendering the newly proposed class of recurrent networks all
	the more interesting. Finally, RACs have recently been shown to operate well in practical settings \citep{khrulkov2018expressive}.
	In the following sections, we make use of the algebraic properties of RACs in
	order to obtain clear-cut observations regarding the benefits of depth in recurrent
	networks.
	
	\ifdefined\SQUEEZE \vspace{-3mm} \fi
	
	\section{Temporal Dependencies Modeled by Recurrent Networks}\label{sec:corr}
	\ifdefined\SQUEEZE \vspace{-2mm} \fi
	
	In this section, we establish means for quantifying the ability of recurrent
	networks to model long-term temporal dependencies in the sequential input data.
	We begin by introducing the Start-End
	separation-rank of the function realized by a
	recurrent network as a measure of the amount of information flow across time
	that can be supported by the network. We then tie the Start-End separation rank to the
	algebraic concept of grid tensors~\citep{hackbusch2012tensor}, which will allow us
	to employ tools and results from tensorial analysis in order to show that depth
	provides a powerful boost to the ability of recurrent networks to model
	elaborate long-term temporal dependencies.
	\ifdefined\SQUEEZE \vspace{-2mm} \fi
	
	\subsection{The Start-End Separation Rank}\label{sec:corr:sep}
	\ifdefined\SQUEEZE \vspace{-2mm} \fi
	
	We define below the concept of the \emph{Start-End separation rank} for
	functions realized by recurrent networks after $T$ time-steps, \ie~functions that take as input $X=(\x^1,\ldots,\x^T)\in\X^T$. The separation rank
	quantifies a function's distance from separability with respect to two
	disjoint subsets of its inputs. Specifically, let $(S,E)$ be a partition of
	input indices, such that $S=\{1,\ldots,\nicefrac{T}{2}\}$ and
	$E=\{\nicefrac{T}{2} + 1, \ldots, T\}$ (we consider even values of $T$
	throughout the paper for convenience of presentation). This implies that
	$\{\x^s\}_{s\in S}$ are the first $T/2$ (``Start'') inputs to the network, and
	$\{\x^e\}_{e\in E}$ are the last $T/2$ (``End'') inputs to the network. For a
	function $y:\X^T\to\R$, the \emph{Start-End separation rank} is defined as
	follows:
	\begin{align}
	\sep{S,E}{y} \equiv\min\left\{K\in\N\cup\{0\}:\exists{g^s_1{\ldots}g^s_K:\X^{T/2}\to\R,~g^e_1{\ldots}g^e_K:\X^{T/2}\to\R}~~s.t.\right.
	\quad~\label{eq:sep_rank}\\
	\left.y(\x^1,\ldots,\x^T) =\sum\nolimits_{\nu=1}^{K}g^s_{\nu}(\x^{1},\ldots,\x^{T/2})g^e_\nu(\x^{T/2+1},\ldots,\x^{T})
	\right\}.
	\nonumber
	\end{align}
	In words, it is the minimal number of summands that together give $y$, where
	each summand is \emph{separable \wrt~$(S,E)$}, \ie~is equal to a product of two
	functions~--~one that intakes only inputs from the first $T/2$ time-steps, and
	another that intakes only inputs from the last $T/2$ time-steps.
	
	The separation rank \wrt~a general partition of the inputs was
	introduced in \citep{beylkin2002numerical} for high-dimensional numerical analysis, and was employed for various
	applications, \eg~chemistry~\citep{harrison2003multiresolution},
	particle engineering~\citep{hackbusch2006efficient}, and machine
	learning~\citep{beylkin2009multivariate}. \citep{cohen2017inductive} connect
	the separation rank to the $L^2$~distance of the function from the set of separable
	functions, and use it to measure dependencies modeled by deep
	convolutional networks. \citep{levine2018deep} tie the separation rank to the
	family of quantum entanglement measures, which quantify dependencies in
	many-body quantum systems.
	
	In our context, if the Start-End separation rank of a function realized by a
	recurrent network is equal to~$1$, then the function is separable, meaning it
	cannot model any interaction between the inputs which arrive at the beginning
	of the sequence and the inputs that follow later, towards the end of the
	sequence. Specifically, if $\sep{S,E}{y}=1$ then there exist
	$g^s:\X^{T/2}\to\R$ and $g^e:\X^{T/2}\to\R$ such that
	$y(\x^1,\ldots,\x^T)=g^s(\x^1,\ldots,\x^{T/2})g^e(\x^{T/2+1},\ldots,\x^{T})$,
	and the function~$y$ cannot take into account consistency between the values of
	$\{\x^1,\ldots,\x^{T/2}\}$ and those of $\{\x^{T/2+1},\ldots,\x^{T}\}$. In a
	statistical setting, if~$y$ were a probability density function, this would
	imply that $\{\x^1,\ldots,\x^{T/2}\}$ and $\{\x^{T/2+1},\ldots,\x^{T}\}$ are
	statistically independent. The higher $\sep{S,E}{y}$ is, the farther~$y$ is
	from this situation, \ie~the more it models dependency between the beginning
	and the end of the inputs sequence. Stated differently, if the recurrent
	network's architecture restricts the hypothesis space to functions with low
	Start-End separation ranks, a more elaborate long-term temporal dependence, which
	corresponds to a function with a higher Start-End separation rank, cannot be
	learned.
	
	In Section~\ref{sec:results} we show that deep RACs support Start-End
	separations ranks which are combinatorially larger than those supported by
	shallow RACs, and are therefore much better fit to model long-term temporal
	dependencies. To this end, we employ in the following sub-section the
	algebraic tool of \emph{grid tensors} that will allow us to evaluate the Start-End separation ranks of deep and
	shallow RACs.
	
	\ifdefined\SQUEEZE \vspace{-2mm} \fi
	
	\subsection{Bounding the Start-End Separation Rank via Grid Tensors}\label{sec:corr:grid}
	\ifdefined\SQUEEZE \vspace{-1mm} \fi
	We begin by laying out basic concepts in tensor theory required for
	the upcoming analysis. The core concept of a \emph{tensor} may be thought of as a
	multi-dimensional array. The \emph{order} of a
	tensor is defined to be the number of indexing entries in the array,
	referred to as \emph{modes}. The \emph{dimension} of a tensor in a particular
	mode is defined as the number of values taken by the index in that
	mode. If $\A$ is a tensor of order $T$ and dimension $M_i$ in each mode
	$i\in[T]$, its entries are denoted $\A_{d_1...d_T}$, where the index in each
	mode takes values $d_i\in [M_i]$.
	A fundamental operator in tensor analysis is the \emph{tensor product}, which
	we denote by $\otimes$. It is an operator that intakes two tensors
	$\A\in\R^{M_1{\times\cdots\times}M_P}$ and
	$\B\in\R^{M_{P+1}{\times\cdots\times}M_{P+Q}}$ (orders $P$ and $Q$ respectively), and returns a tensor
	$\A\otimes\B\in\R^{M_1{\times\cdots\times}M_{P+Q}}$  (order $P+Q$) defined by:
	$(\A\otimes\B)_{d_1{\ldots}d_{P+Q}}=\A_{d_1{\ldots}d_P}\cdot\B_{d_{P+1}{\ldots}d_{P+Q}}$.
	An additional concept we will make use of is the \emph{matricization of $\A$ \wrt~the partition $(S,E)$}, denoted $\mat{\A}_{S,E}\in\R^{M^{\nicefrac{T}{2}}\times M^{\nicefrac{T}{2}}}$, which is essentially the arrangement of the tensor elements as a matrix whose rows correspond to $S$ and columns to $E$ (formally presented in Appendix~\ref{app:mat}).
	
	We consider the
	function realized by a shallow RAC with $R$ hidden channels, which
	computes the score of class $c\in[C]$ at time $T$. This function, which is
	given by a recursive definition in Equations~\eqref{eq:shallow_rn} and~\eqref{eq:g_rac},
	can be alternatively written in the following closed form:
	\begin{equation}
	y^{T,1, \Theta}_{c}\left(\x^1,\ldots,\x^T\right)
	= \sum\nolimits_{d_1{\ldots}d_T=1}^M\left(\A^{T,1, \Theta}_c\right)_{d_1,\ldots,d_T}\prod\nolimits_{i=1}^{T} f_{d_i}(\x^i),
	\label{eq:score}
	\end{equation}
	where the order $T$ tensor $\A^{T,1, \Theta}_c$, which lies at the heart of
	the above expression, is referred to as the \emph{shallow RAC weights tensor},
	since its entries are polynomials in the network weights $\Theta$.
	Specifically, denoting the rows of the input weights matrix, $\WInone$,
	by $\aaa^{\mathrm{I},\alpha}\in\R^M$ (or element-wise:
	$a^{\mathrm{I},\alpha}_j=\WInone_{\alpha,j}$), the rows of the hidden
	weights matrix, $\WHnone$, by $\aaa^{\mathrm{H},\beta}\in\R^R$ (or
	element-wise: $a^{\mathrm{H},\beta}_j=\WHnone_{\beta,j}$), and the rows
	of the output weights matrix, $\WO$, by
	$\aaa^{\mathrm{O},c}\in\R^R,~c\in[C]$ (or element-wise:
	$a^{\mathrm{O},c}_j=\WO_{c,j}$), the shallow RAC weights tensor can
	be gradually constructed in the following fashion:
	\begin{align}
	\underbrace{\phi^{2,\beta}}_{\text{order $2$ tensor}} &= \sum\nolimits_{\alpha=1}^{R}~ a_\alpha^{\mathrm{H},\beta}~~
	\aaa^{\mathrm{I},\alpha} ~\otimes \aaa^{\mathrm{I},\alpha}
	\nonumber \\[-3mm]
	&~~~\cdots
	\nonumber\\[-1mm]
	\underbrace{\phi^{t,\beta}}_{\text{order $t$ tensor}} &= \sum\nolimits_{\alpha=1}^{R} a_\alpha^{\mathrm{H},\beta}
	\phi^{t-1,\alpha} \otimes
	\aaa^{\mathrm{I},\alpha}
	\nonumber\\[-3mm]
	&~~~\cdots
	\nonumber\\[-1mm]
	\underbrace{\A^{T,1, \Theta}_c}_{\text{order $T$ tensor}} &= \sum\nolimits_{\alpha=1}^{R}
	a_\alpha^{\mathrm{O},c}
	\phi^{T-1,\alpha} \otimes
	\aaa^{\mathrm{I},\alpha},
	\label{eq:tt_decomp}
	\end{align}
	having set $\h^0=\left(W^{\textrm H}\right)^{\dagger}\1$, where $\dagger$ is the pseudoinverse operation. In the above equation, the tensor products, which appear inside the sums, are
	directly related to the Multiplicative Integration property of RACs
	[Equation~\eqref{eq:g_rac}]. The sums originate in the multiplication of the hidden
	states vector by the hidden weights matrix at every time-step
	[Equation~\eqref{eq:shallow_rn}].
	The construction of the shallow RAC weights tensor, presented in Equation~\eqref{eq:tt_decomp}, is referred to as a Tensor Train (TT) decomposition of
	TT-rank $R$ in the tensor analysis community~\citep{oseledets2011tensor} and is analogously described by a
	Matrix Product State (MPS) Tensor Network (see \citep{orus2014practical}) in the quantum physics community.
	See Appendix~\ref{app:rac_tns} for the Tensor Networks construction of deep and shallow RACs, which provides graphical insight regarding the complexity brought forth by depth in recurrent networks.

	We now present the concept of grid tensors, which are a form of function discretization. Essentially, the function is
	evaluated for a set of points on an exponentially large grid in the
	input space and the outcomes are stored in a tensor. Formally, fixing a set of \emph{template} vectors
	$\x^{(1)},\ldots,\x^{(M)} \in \X$, the points on the grid are the set
	$\{(\x^{(d_1)},\ldots,\x^{(d_T)})\}_{d_1,\ldots,d_T=1}^M$. Given a function
	$y(\x^1,\ldots,\x^T)$, the set of its values on the grid arranged in the form of a tensor are
	called the grid tensor induced by $y$, denoted
	$\A(y)_{d_1,\ldots,d_T} \equiv y(\x^{(d_1)},\ldots,\x^{(d_T)})$.
	The grid tensors of functions realized by recurrent networks, will allow us to
	calculate their separations ranks and establish definitive conclusions
	regarding the benefits of depth these networks.
	Having presented the tensorial structure of the function realized by a shallow
	RAC, as given by Equations~\eqref{eq:score} and~\eqref{eq:tt_decomp} above, we are now
	in a position to tie its Start-End separation rank to its grid tensor, as
	formulated in the following claim:
	\begin{claim}\label{claim:grid_sep_shallow}
		Let $y^{T,1, \Theta}_c$ be a function realized by a shallow RAC
		(top of Figure~\ref{fig:recurrent_net}) after $T$
		time-steps, and let $\A^{T,1, \Theta}_c$ be its shallow RAC weights tensor,
		constructed according to Equation~\eqref{eq:tt_decomp}.
		Assume that the network's initial mapping functions~$\{f_{d}\}_{d=1}^M$ are
		linearly independent, and that they, as well as the
		functions~$g^s_\nu,g^e_\nu$ in the definition of Start-End separation rank
		[Equation~\eqref{eq:sep_rank}], are measurable and square-integrable.
		Then, there exist template vectors $\x^{(1)}, \ldots, \x^{(M)} \in \X$
		such that the following holds:
		\begin{equation}
		\sep{S,E}{y^{T,1, \Theta}_c}=\rank{\mat{\A(y^{T,1, \Theta}_c)}_{S,E}}=\rank{\mat{\A^{T,1, \Theta}_c})_{S,E}},
		\end{equation}
		where $\A(y^{T,1, \Theta}_c)$ is the grid tensor of $y^{T,1, \Theta}_c$
		with respect to the above template vectors.
	\end{claim}
	\ifdefined\SQUEEZE \vspace{-4mm}\fi
	
	\begin{proof}
		We first note that though square-integrability may seem as a limitation at first glance (for
		example neurons $f_{d}(\x)=\sigma(\w_d^\top\x+b_d)$ with sigmoid or ReLU
		activation $\sigma(\cdot)$, do not meet this condition), in practice our inputs are bounded (\eg~image pixels by holding
		intensity values, etc). Therefore, we may view these functions as having compact support,
		which, as long as they are continuous (holds in all cases of interest),
		ensures square-integrability.
		
		We begin by proving the equality
		$\sep{S,E}{y_c^{T,1, \Theta}} = \rank{\mat{\A_c^{T,1, \Theta}}_{S,E}}$. As shown in
		\citep{cohen2017inductive}, for any function
		$f:{\cal{X}}\times \cdots \times {\cal{X}} \to R$ which follows the structure of
		Equation~\eqref{eq:score} with a general weights tensor $\A$, assuming that
		$\{f_d\}_{d=1}^M$ are linearly independent, measurable, and
		square-integrable (as assumed in Claim~\ref{claim:grid_sep_shallow}), it
		holds that $\sep{S,E}{f} = \rank{\mat{\A}_{S,E}}$. Specifically, for
		$f=y_c^{T,1, \Theta}$ and $\A = \A_c^{T,1, \Theta}$ the above equality holds.
		
		It remains to prove that there exists template vectors for which
		$\rank{\mat{\A_c^{T,1, \Theta}}_{S,E}} = \rank{\mat{\A(y_c^{T,1, \Theta})}_{S,E}}$.
		For any given set of template vectors $\x^{(1)},\ldots,\x^{(M)}\in\X$, we define
		the matrix $F \in \R^{M \times M}$ such that $F_{ij} = f_j(\x^{(i)})$, for which
		it holds that:
		\begin{align*}
		\A(y_c^{T,1, \Theta})_{k_1,\ldots,k_T} &= \sum_{d_1,\ldots,d_T = 1}^M \left(\A_c^{T,1, \Theta} \right)_{d_1,\ldots,d_T} \prod_{i=1}^T f_{d_i}(\x^{(k_i)}) \\
		&= \sum_{d_1,\ldots,d_T = 1}^M \left(\A_c^{T,1, \Theta} \right)_{d_1,\ldots,d_T} \prod_{i=1}^T F_{k_i d_i}.
		\end{align*}
		The right-hand side in the above equation can be regarded as a linear
		transformation of $\A_c^{T,1, \Theta}$ specified by the tensor operator
		$F \otimes \cdots \otimes F$, which is more commonly denoted by
		$(F \otimes \cdots \otimes F)(\A_c^{T,1, \Theta})$. According to lemma 5.6 in
		\citep{hackbusch2012tensor}, if $F$ is non-singular then
		$\rank{\mat{(F \otimes \cdots \otimes F)(\A_c^{T,1, \Theta})}_{S,E}} =
		\rank{\mat{\A_c^{T,1, \Theta}}_{S,E}}$. To conclude
		the proof, we simply note that \citep{cohen2016convolutional}
		showed that if $\{f_d\}_{d=1}^M$ are linearly independent then there exists
		template vectors for which $F$ is non-singular.
	\end{proof}
	\vspace{2mm}
	
	The above claim establishes an equality between the Start-End separation rank
	and the rank of the matrix obtained by the corresponding grid tensor
	matricization, denoted $\mat{\A(y^{T,1, \Theta}_c)}_{S,E}$, with respect to a
	specific set of template vectors. Note that the limitation to specific
	template vectors does not restrict our results, as grid tensors are
	merely a tool used to bound the separation rank. The additional equality to the rank of the
	matrix obtained by matricizing the shallow RAC weights tensor, will be of use to
	us when proving our main results below (Theorem~\ref{theorem:main_result}).
	
	Due to the inherent use of data duplication in the computation preformed by a
	deep RAC (see Appendix~\ref{app:rac_tns:deep} for further details), it cannot be
	written in a closed tensorial form similar to that of Equation~\eqref{eq:score}.
	This in turn implies that the equality shown in
	Claim~\ref{claim:grid_sep_shallow} does not hold for functions realized by deep
	RACs. The following claim introduces a fundamental relation between a function's
	Start-End separation rank and the rank of the matrix obtained by the
	corresponding grid tensor matricization. This relation, which holds for all functions, is
	formulated below for functions realized by deep RACs:
	\begin{claim}\label{claim:grid_sep_deep}
		Let $y^{T,L, \Theta}_c$ be a function realized by a depth $L$ RAC
		(bottom of Figure~\ref{fig:recurrent_net}) after $T$
		time-steps. Then, for any set of template vectors $\x^{(1)},\ldots,\x^{(M)} \in \X$ it
		holds that:
		\begin{equation}
		\sep{S,E}{y^{T,L, \Theta}_c}\geq \rank{\mat{\A(y^{T,L, \Theta}_c)}_{S,E}},
		\end{equation}
		where $\A(y^{T,L, \Theta}_c)$ is the grid tensor of $y^{T,L, \Theta}_c$ with
		respect to the above template vectors.
	\end{claim}
	\ifdefined\SQUEEZE \vspace{-4mm} \fi
	\begin{proof}
		If $\sep{S,E}{y_c^{T,L, \Theta}} = \infty$ then the inequality is trivially
		satisfied. Otherwise, assume that
		$\sep{S,E}{y_c^{T,L, \Theta}} = K \in \N$, and let $\{g_i^s, g_i^e\}_{i=1}^K$
		be the functions of the respective decomposition to a sum of separable
		functions, i.e. that the following holds:
		\begin{align*}
		y_c^{T,L, \Theta}(\x^1,\ldots,\x^T)
		&= \sum_{\nu=1}^K g_\nu^s(\x^1,\ldots,\x^{T/2})
		\cdot g_\nu^e(\x^{T/2+1},\ldots,\x^T).
		\end{align*}
		Then, by definition of the grid tensor, for any template vectors $\x^{(1)},\ldots,\x^{(M)}\in \X$ the following
		equality holds:
		\begin{align*}
		\A(y_c^{T,L, \Theta})_{d_1,\ldots,d_N} &=
		\sum_{\nu = 1}^K g_\nu^s(\x^{(d_1)},\ldots,\x^{(d_{T/2})})
		\cdot g_\nu^e(\x^{(d_{T/2+1})},\ldots,\x^{(d_T)}) \\
		&\equiv \sum_{\nu=1}^K V^\nu_{d_1,\ldots,d_{T/2}} U^\nu_{d_{T/2+1},\ldots,d_T},
		\end{align*}
		where $V^\nu$ and $U^\nu$ are the tensors holding the values of $g_\nu^s$
		and $g_\nu^e$, respectively, at the points defined by the template vectors.
		Under the matricization according to the $(S,E)$ partition, it holds that
		$\mat{V^\nu}_{S,E}$ and $\mat{U^\nu}_{S,E}$ are column and row vectors,
		respectively, which we denote by $\vv_\nu$ and $\uu_\nu^T$. It follows that the
		matricization of the grid tensor is given by:
		\begin{align*}
		\mat{\A(y_c^{T,L, \Theta})}_{S,E} &= \sum_{\nu=1}^K \vv_\nu \uu_\nu^T,
		\end{align*}
		which means that
		$\rank{\mat{\A(y_c^{T,L, \Theta})}_{S,E}}\leq K=\sep{S,E}{y_c^{T,L, \Theta}}$.
	\end{proof}
	\vspace{2mm}
	
	Claim~\ref{claim:grid_sep_deep} will allow us to provide a lower bound on the
	Start-End separation rank of functions realized by deep RACs, which we show to
	be significantly higher than the Start-End separation rank of functions realized
	by shallow RACs (to be obtained via Claim~\ref{claim:grid_sep_shallow}). Thus, in the next section, we employ the above presented tools to
	show that a compelling enhancement of the Start-End separation
	rank is brought forth by depth in recurrent networks.
	\ifdefined\SQUEEZE\vspace{-3mm}\fi
	\section{Depth Enhanced Long-Term Memory in Recurrent Networks}
	\ifdefined\SQUEEZE\vspace{-2mm}\fi
	
	\label{sec:results}

	In this section, we present the main theoretical contributions of this paper.
	In Section~\ref{sec:results:proof}, we formally present a result which clearly separates between the memory capacity of a deep ($L=2$) recurrent network and a shallow ($L=1$) one.
	Following the formal presentation of results in
	Theorem~\ref{theorem:main_result}, we discuss some of their implications and
	then conclude by sketching a proof outline for the theorem (full proof is
	relegated to Appendix~\ref{app:proofs:main_result}). In Section~\ref{sec:results:conjecture}, we present a quantitative conjecture regarding the enhanced memory capacity of deep recurrent networks of general depth $L$, which relies on the inherent combinatorial properties of the recurrent network's computation. We leave the formal proof of this conjecture for future work.
	
	\ifdefined\SQUEEZE\vspace{-2mm}\fi
	\subsection{Separating Between Shallow and Deep Recurrent Networks}\label{sec:results:proof}
	\ifdefined\SQUEEZE\vspace{-2mm}\fi
	
	Theorem~\ref{theorem:main_result} states, that the dependencies modeled between
	the beginning and end of the input sequence to a recurrent network, as measured
	by the Start-End separation rank (see Section~\ref{sec:corr:sep}), can be considerably
	more complex for deep networks than for shallow ones:
	
	\begin{theorem}\label{theorem:main_result}
		Let $y^{T,L, \Theta}_c$ be the function computing the output after
		$T$ time-steps of an RAC with $L$ layers, $R$ hidden channels per layer,
		weights denoted by $\Theta$, and initial hidden states $\h^{0,l},~l\in[L]$ (Figure~\ref{fig:recurrent_net} with
		$g=g^{\mathrm {RAC}}$). Assume that the network's initial mapping functions~$\{f_{d}\}_{d=1}^M$ are
		linearly independent and square integrable. Let $\sep{S,E}{y^{T,L, \Theta}_c}$ be the Start-End
		separation rank of $y^{T,L, \Theta}_c$ [Equation~\eqref{eq:sep_rank}]. Then, the
		following holds almost everywhere, \ie~for all values of the parameters $\Theta\times\h^{0,l}$ but a set of
		Lebesgue measure zero:
		
		\begin{enumerate}
			\ifdefined\SQUEEZE \vspace{-2mm}\fi
			\item $\sep{S,E}{y^{T,L, \Theta}_c} = \min\left\{R, M^{\nicefrac{T}{2}} \right\}$, for $L=1$ (shallow network).
			\ifdefined\SQUEEZE	\vspace{-3mm}\fi
			\item $\sep{S,E}{y^{T,L, \Theta}_c} \geq~\multiset{\min\{M,R\}}{T/2} $, for $L =2$ (deep network),
		\end{enumerate}
		\ifdefined\SQUEEZE	\vspace{-3mm}\fi
		where $\multiset{\min\{M,R\}}{T/2}$ is the multiset coefficient, given in the binomial form by
		$\binom{\min\{M,R\}+T/2-1}{T/2}$.
	\end{theorem}
	\ifdefined\SQUEEZE\vspace{-2mm}\fi
	
	The above theorem readily implies that depth entails an enhanced ability of
	recurrent networks to model long-term temporal dependencies in the sequential
	input.
	Specifically, Theorem~\ref{theorem:main_result} indicates depth efficiency -- it
	ensures us that upon randomizing the weights of a deep RAC with $R$ hidden
	channels per layer, with probability $1$ the function realized by it after $T$
	time-steps may only be realized by a shallow RAC with a number of hidden
	channels that is combinatorially large. Stated alternatively, this means that
	almost all functional dependencies which lie in the hypothesis space of deep
	RACs with $R$ hidden channels per layer, calculated after $T$ time-steps, are
	inaccessible to shallow RACs with less than a super-linear number of
	hidden channels. Thus, a shallow recurrent network would require an impractical
	amount of parameters if it is to implement the same
	function as a deep recurrent network.
	
	The established role of the Start-End separation rank as a dependency measure between the
	beginning and the end of the sequence (see Section~\ref{sec:corr:sep}), implies
	that these functions, which are realized by almost any deep network and can
	never be realized by a shallow network of a reasonable size, represent more
	elaborate dependencies over longer periods of time.
	The above notion is strengthened by the fact that the Start-End separation rank
	of deep RACs increases with the sequence length $T$, while the Start-End
	separation rank of shallow RACs is independent of it.
	This indicates that shallow recurrent networks are much more restricted in
	modeling long-term dependencies than the deep ones, which enjoy a combinatorially
	increasing Start-End separation rank as time progresses.
	Below, we present an outline of the proof for Theorem~\ref{theorem:main_result}
	(see Appendix~\ref{app:proofs:main_result} for the full proof):
	
	\vspace{2mm}
	\ifdefined\SQUEEZE\vspace{-3mm}\fi
	\begin{proof}[Proof sketch of Theorem~\ref{theorem:main_result}]
		\leavevmode
		\ifdefined\SQUEEZE	\vspace{-1mm}\fi
		\begin{enumerate}
			\item For a shallow network, Claim~\ref{claim:grid_sep_shallow} establishes
			that the Start-End separation rank of the function realized by a shallow
			($L=1$) RAC is equal to the rank of the matrix obtained by matricizing
			the corresponding shallow RAC weights tensor [Equation~\eqref{eq:score}]
			according to the Start-End partition: $\sep{S,E}{y^{T,1, \Theta}_c}= \rank{\mat{\A^{T,1, \Theta}_c})_{S,E}}$.
			Thus, it suffices to prove that $\rank{\mat{\A^{T,1, \Theta}_c})_{S,E}}=R$
			in order to satisfy bullet
			(1) of the theorem, as the rank is trivially upper-bounded by the
			dimension of the matrix, $M^{\nicefrac{T}{2}}$. To this end, we call
			upon the TT-decomposition of $\A^{T,1, \Theta}_c$, given by
			Equation~\eqref{eq:tt_decomp}, which
			corresponds to the MPS Tensor Network presented in Appendix~\ref{app:rac_tns}. We rely on a recent result by \citep{levine2018deep},
			who state that the rank of the matrix obtained by matricizing any tensor
			according to a partition $(S,E)$, is equal to a minimal cut separating $S$
			from $E$ in the Tensor Network graph representing this tensor. The
			required equality follows from the fact that the TT-decomposition in
			Equation~\eqref{eq:tt_decomp} is of TT-rank $R$, which in turn implies that the
			min-cut in the appropriate Tensor Network graph is equal to $R$.
			\item For a deep network, Claim~\ref{claim:grid_sep_deep} assures us that the
			Start-End separation rank of the function realized by a depth $L=2$ RAC
			is lower bounded by the rank of the matrix obtained by the corresponding
			grid tensor matricization: $\sep{S,E}{y^{T,2, \Theta}_c}\geq \rank{\mat{\A(y^{T,2, \Theta}_c)}_{S,E}}$.
			Thus, proving that
			$\rank{\mat{\A(y^{T,2, \Theta}_c)}_{S,E}} \geq \multiset{\min\{M,R\}}{T/2}$ for
			all of the values of parameters $\Theta\times\h^{0,l}$ but a set of Lebesgue measure
			zero, would satisfy the theorem.
			We use a lemma proved in
			\citep{sharirtractable}, which states that since the entries of
			$\A(y^{T,2, \Theta}_c)$ are polynomials in the deep recurrent network's
			weights, it suffices to find a single example for which the rank of the
			matricized grid tensor is greater than the desired lower bound. Finding
			such an example would indeed imply that for almost all of the values of
			the network parameters, the desired inequality holds.
			
			We choose a weight assignment such that the resulting matricized grid
			tensor resembles a matrix obtained by raising a rank-$\bar{R}\equiv\min\{M,R\}$ matrix to the
			Hadamard power of degree $T/2$. This operation, which raises each
			element of the original rank-$\bar{R}$ matrix to the power of $T/2$, was shown
			to yield a matrix with a rank upper-bounded by the multiset
			coefficient $\multiset{\bar{R}}{T/2}$ (see \eg~\citep{amini2012low}). We show
			that our assignment results in a matricized grid tensor with a rank
			which is not only upper-bounded by this value, but actually achieves it. Under our assignment, the matricized grid tensor takes the form: 
			\begin{equation*}
			\mat{\A(y^{T,2, \Theta}_c)}_{S,E}=
			\sum_{\substack{\p\in \state{{\bar{R}}}{\nicefrac{T}{2}}}}
			U_{S\p} \cdot
			V_{\p E},
			\end{equation*}
			where the set $\state{{\bar{R}}}{\nicefrac{T}{2}} \equiv \{\p \in (\N \cup \{0\})^{\bar{R}} |\sum_{i=1}^{\bar{R}} p_i = \nicefrac{T}{2}\}$ can be viewed as the set of all possible states of a bucket containing $\nicefrac{T}{2}$
			balls of ${\bar{R}}$ colors, where $p_r$ for $r\in[\bar{R}]$ specifies the number of balls
			of the $r$'th color. By definition: $\abs{\state{{\bar{R}}}{\nicefrac{T}{2}}}=\multiset{\bar{R}}{\nicefrac{T}{2}}$ and $\abs{S}=\abs{E}=M^{\nicefrac{T}{2}}$, therefore the matrices $U$ and $V$ uphold: $U\in\R^{M^{\nicefrac{T}{2}}\times\multiset{\bar{R}}{\nicefrac{T}{2}}}$; $V\in\R^{\multiset{\bar{R}}{\nicefrac{T}{2}}\times M^{\nicefrac{T}{2}}}$, and for the theorem to follow we must show that they both are of rank $\multiset{\bar{R}}{\nicefrac{T}{2}}$ (note that $\multiset{\bar{R}}{\nicefrac{T}{2}}\leq\multiset{M}{\nicefrac{T}{2}}<M^{\nicefrac{T}{2}}$).
			
			We observe the sub-matrix $\bar{U}$ defined by the subset of the rows of $U$ such that we select the row $d_1\ldots d_{\nicefrac{T}{2}}\in S$ only if it upholds
			that $\forall j, d_j \leq d_{j+1}$. Note that there are exactly $\multiset{\bar{R}}{\nicefrac{T}{2}}$ such rows, thus $\bar{U}\in\R^{\multiset{\bar{R}}{\nicefrac{T}{2}}\times\multiset{\bar{R}}{\nicefrac{T}{2}}}$ is a square matrix.
			Similarly we observe a sub-matrix of $V$ denoted $\bar{V}$, for which we select the column $d_{\nicefrac{T}{2}+1}\ldots d_T \in E$ only if it upholds
			that $\forall j, d_j \leq d_{j+1}$, such that it is also a square matrix. Finally, by employing a variety of technical lemmas, we show that the determinants of these square matrices are non vanishing under the given assignment, thus satisfying the theorem. 
			
		\end{enumerate}
		\ifdefined\SQUEEZE	\vspace{-3mm}\fi
	\end{proof}

	\subsection{Increase of Memory Capacity with Depth}\label{sec:results:conjecture}
	\ifdefined\SQUEEZE\vspace{-1mm}\fi
	
	Theorem~\ref{theorem:main_result} provides a lower bound of
	$\multiset{R}{\nicefrac{T}{2}}$ on the Start-End separation rank of depth $L=2$
	recurrent networks, combinatorially separating deep recurrent networks from shallow ones.
	By a trivial assignment of weights in higher layers, the Start-End separation rank of even
	deeper recurrent networks ($L>2$) is also lower-bounded by this expression, which does not
	depend on $L$. In the following, we conjecture that a tighter lower bound holds for
	networks of depth $L>2$, the form of which implies that the memory capacity of deep recurrent networks
	grows combinatorially with the network depth:
	
	\begin{conjecture}\label{conjecture:high_L}
		Under the same conditions as in Theorem~\ref{theorem:main_result}, for all values of
		$\Theta\times\h^{0,l}$ but a set of Lebesgue measure zero, it holds for any $L$ that:
		\ifdefined\SQUEEZE	\vspace{-2mm}\fi
		\begin{align*}
		\sep{S,E}{y^{T,L, \Theta}_c}&\geq \min\left\{\multiset{\min\{M,R\}}{\multiset{T/2}{L-1}}, M^{\nicefrac{T}{2}} \right\}.
		\end{align*}
	\end{conjecture}
	
	We motivate Conjecture~\ref{conjecture:high_L} by investigating the combinatorial nature
	of the computation performed by a deep RAC. By constructing Tensor Networks which correspond to deep RACs, we attain an informative visualization of this combinatorial perspective.
	In Appendix~\ref{app:rac_tns}, we provide full details of this Tensor Networks construction and present
	the formal motivation for the conjecture in Appendix~\ref{app:rac_tns:conjecture}. Below, we qualitatively outline
	our approach.
	
	A Tensor Network is essentially a graphical tool for representing algebraic operations
	which resemble multiplications of vectors and matrices, between higher order tensors.
	Figure~\ref{fig:conjecture} (top) shows an example of the Tensor Network representing the computation of a depth $L=3$ RAC after
	$T=6$ time-steps. This well-defined computation graph hosts the values of the weight
	matrices at its nodes. The inputs $\{x^1,\ldots, x^T\}$ are marked by their corresponding time-step $\{1,\ldots, T\}$, and are integrated in a depth dependent and time-advancing manner (see further discussion regarding this form in Appendix~\ref{app:rac_tns:deep}), as portrayed in the example of Figure~\ref{fig:conjecture}. We highlight in red the basic unit in
	the Tensor Network which connects ``Start" inputs $\{1,\ldots,\nicefrac{T}{2}\}$ and ``End" inputs $\{\nicefrac{T}{2}+1,\ldots,T\}$. In order to estimate
	a lower bound on the Start-End separation rank of a depth $L>2$ recurrent network, we employ a
	similar strategy to that presented in the proof sketch of the $L=2$ case (see
	Section~\ref{sec:results:proof}). Specifically, we rely on the fact that it is sufficient
	to find a specific instance of the network parameters $\Theta\times\h^{0,l}$ for which
	$\mat{\A(y^{T,L, \Theta}_c)}_{S,E}$ achieves a certain rank, in order for this rank to
	bound the Start-End separation rank of the network from below.
	
	\begin{figure}
		\centering
		\includegraphics[width=1\linewidth]{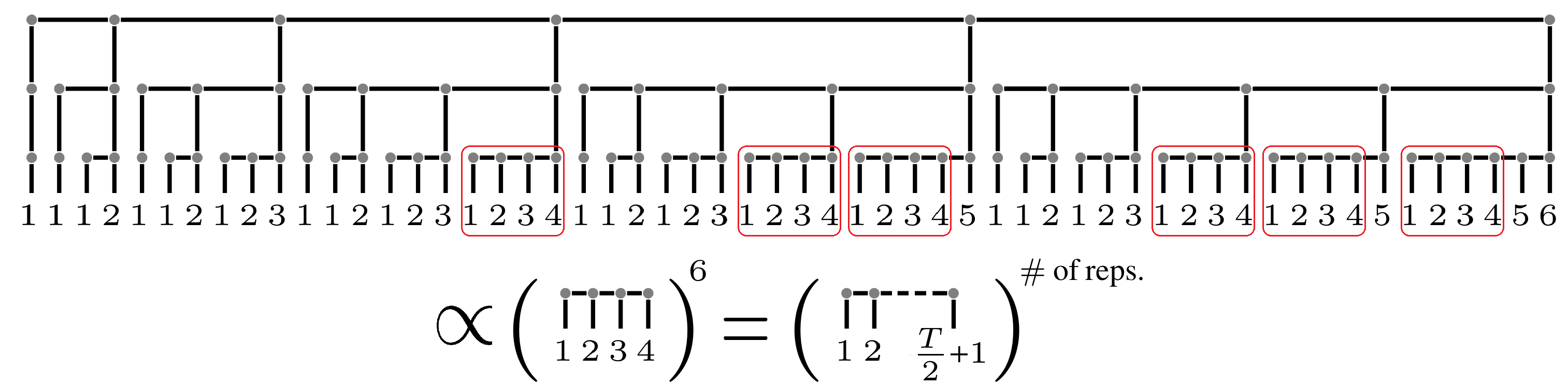}
		\caption{Tensor Network representing the computation of a depth $L=3$ RAC after $T=6$ time-steps. See construction in Appendix~\ref{app:rac_tns}. The number of repetitions of the basic unit cell connecting `Start' and `End' inputs in the Tensor Network graph gives rise to the lower bound in Conjecture~\ref{conjecture:high_L}.}
		\label{fig:conjecture}
		\ifdefined\SQUEEZE	\vspace{-5mm}\fi
	\end{figure}
	
	Indeed, we find a specific assignment of the network weights, presented in
	Appendix~\ref{app:rac_tns:conjecture}, for which the Tensor Network effectively takes the form of the
	basic unit connecting ``Start" and ``End", raised to the power of the number of its
	repetitions in the graph (bottom of Figure~\ref{fig:conjecture}). This basic unit corresponds to a simple computation
	represented by a grid tensor with Start-End matricization of rank $R$. Raising such a
	matrix to the Hadamard power of any $p\in\mathbb{Z}$, results in a matrix with a rank
	upper bounded by $\multiset{R}{p}$, and the challenge of
	proving the conjecture amounts to proving that the upper bound is tight in this case.
	In Appendix~\ref{app:rac_tns:conjecture}, we prove that the number of repetitions of the basic unit connecting ``Start" and ``End" in the deep RAC Tensor Network graph, is exactly equal to
	$\multiset{T/2}{L-1}$ for any depth $L$. For example, in the $T=6,L=3$
	network illustrated in Figure~\ref{fig:conjecture}, the number of repetitions indeed corresponds to
	$p=\multiset{3}{2}=6$. It is noteworthy that for $L=1,2$~ the bound in
	Conjecture~\ref{conjecture:high_L} coincides with the bounds that were proved for these depths in
	Theorem~\ref{theorem:main_result}.
	
	Conjecture~\ref{conjecture:high_L} indicates that beyond the proved combinatorial advantage
	in memory capacity of deep networks over shallow ones, a further combinatorial separation
	may be shown between recurrent networks of different depths. We leave the proof of this
	result, which can reinforce and refine the understanding of advantages brought forth by
	depth in recurrent networks, as an open problem.
	In the following, we empirically investigate the theoretical outcomes presented in this section.
	
	\section{Experiments}\label{sec:exp}
	
	In this section, we provide an empirical demonstration supporting the theoretical findings of this paper. 
	The results above are formulated for the class of RACs (presented in Section~\ref{sec:racs}), and the experiments presented hereinafter demonstrate their extension to more commonly used RNN architectures. 
	As noted in Section~\ref{sec:intro}, the advantage of deep recurrent networks over shallow ones is well established empirically, as the best results on various sequential tasks have been achieved by stacking recurrent layers~\citep{cirecsan2010deep,mohamed2012acoustic,graves2013speech}. 
	Below, we focus on two tasks which highlight the `long-term memory' demand of recurrent networks, and show how depth empowers the network's ability to express the appropriate distant temporal dependencies. 
	
	We address two synthetic problems. The first is the Copying Memory Task, to be
	described in Section~\ref{sec:exp:copy}, which was previously used to test
	proposed solutions to the gradient issues of backpropagation through
	time~\citep{hochreiter1997long,martens2011learning,arjovsky2016unitary,wisdom2016full,jing2016tunable}.
	We employ this task as a test for the recurrent network's expressive ability to
	`remember' information seen in the distant past. 
	The second task is referred to as the Start-End Similarity Task, to be described
	in Section~\ref{sec:exp:sim}, which is closely related to the Start-End
	separation rank measure proposed in Section~3.1.
	In both experiments we use a successful RNN variant referred to as Efficient
	Unitary Recurrent Neural Network (EURNN)~\citep{jing2016tunable}, which was
	shown to enable efficient optimization without the need to use gating units such
	as in LSTM networks to overcome the vanishing gradient problem. Moreover, EURNNs
	are known to perform exceptionally well on the Copying Memory Task.
	Specifically, we use EURNN in its most basic form, with orthogonal
	hidden-to-hidden matrices, and with the tunable parameter (see
	\citep{jing2016tunable} Section 4.2) set to 2. Under the notations we presented
	in section~\ref{sec:racs} and portrayed in Figure~\ref{fig:recurrent_net},
	EURNNs employ $g^\mathrm{RNN}(\aaa, \bb) = \sigma(\aaa + \bb)$, where
	$\sigma(\cdot)$ is the modReLU function (\citep{jing2016tunable} Section 4.5),
	and the matrices $W^{\mathrm{H},l}$ are restricted to being orthogonal.
	Throughout both experiments we use RMSprop~\citep{tieleman2012lecture} as the
	optimization algorithm, where we took the best of several moving average discount factor values between	$0.5$ (in accordance with~\citep{jing2016tunable}) and the default value of $0.9$, and with a learning rate of $10^{-3}$. We use a training set of size 100K,
	a test set of size 10K, and a mini-batch size of $128$. 
	
	The methodology we employ in the experiments below is aimed at testing the
	following practical hypothesis, which is commensurate with the theoretical
	outcomes in Section~\ref{sec:results}: \emph{Given a certain resource budget for
		a recurrent network that is intended to solve a `long-term memory problem',
		adding recurrent layers is significantly preferable to increasing the number of
		channels in existing layers}. Specifically, we train RNNs of depths $1$, $2$, and
	$3$ over increasingly hard variants of each problem (requiring longer-term
	memory), and report the maximal amount of memory capabilities for each
	architecture in Figures~\ref{fig:copy} and~\ref{fig:sim}. 
	
	\subsection{Copying Memory Task}\label{sec:exp:copy}
	
	In the Copying Memory Task, the network is required to
	memorize a sequence of characters of fixed length $m$,
	and then to reproduce it after a long lag of $B$
	time-steps, known as the \emph{delay time}.
	The input sequence is composed of characters drawn from a given
	alphabet $\left\{ a_{i}\right\} _{i=1}^{n}$, and two
	special symbols: a \emph{blank} symbol denoted by
	`\texttt{\_}', and a \emph{trigger} symbol denoted by
	`\texttt{:}'. The input begins with a string of $m$
	\emph{data characters} randomly drawn from the alphabet,
	and followed by $B$ occurrences of blank symbols.
	On the $m$'th before last time-step the trigger symbol
	is entered, signaling that the data needs to be
	presented. Finally the input ends with an additional
	$m-1$ blank characters. In total, the sequence length
	is $T=B+2m$. The correct sequential output of this task
	is referred to as the target. The target character in
	every time-step is always the blank character, except
	for the last $m$ time-steps, in which the target is the
	original $m$ data characters of the input. For example,
	if $m=3$ and $B=5$, then a legal input-output pair could
	be ``\texttt{ABA\_\_\_\_\_:\_\_}'' and
	``\texttt{\_\_\_\_\_\_\_\_ABA}'', respectively.

	In essence, the data length $m$ and alphabet size $n$
	control the number of bits to be memorized, and the delay
	time controls the time these bits need to stay in
	memory~--~together these parameters control the hardness
	of the task. Previous works have used values such as
	$m=10$ and $n=8$~\citep{arjovsky2016unitary} or similar,
	which amount to memorizing $30$ bits of information, for
	which it was demonstrated that even shallow recurrent
	networks are able to solve this task for delay times as
	long as $B=1000$ or more. To allow us to properly
	separate between the performance of networks of different
	depths, we consider a much harder variant with $m=30$ and
	$n=32$, which requires memorizing $150$ bits of
	information.
	
	\begin{figure}
		\centering
		\includegraphics[width=\linewidth]{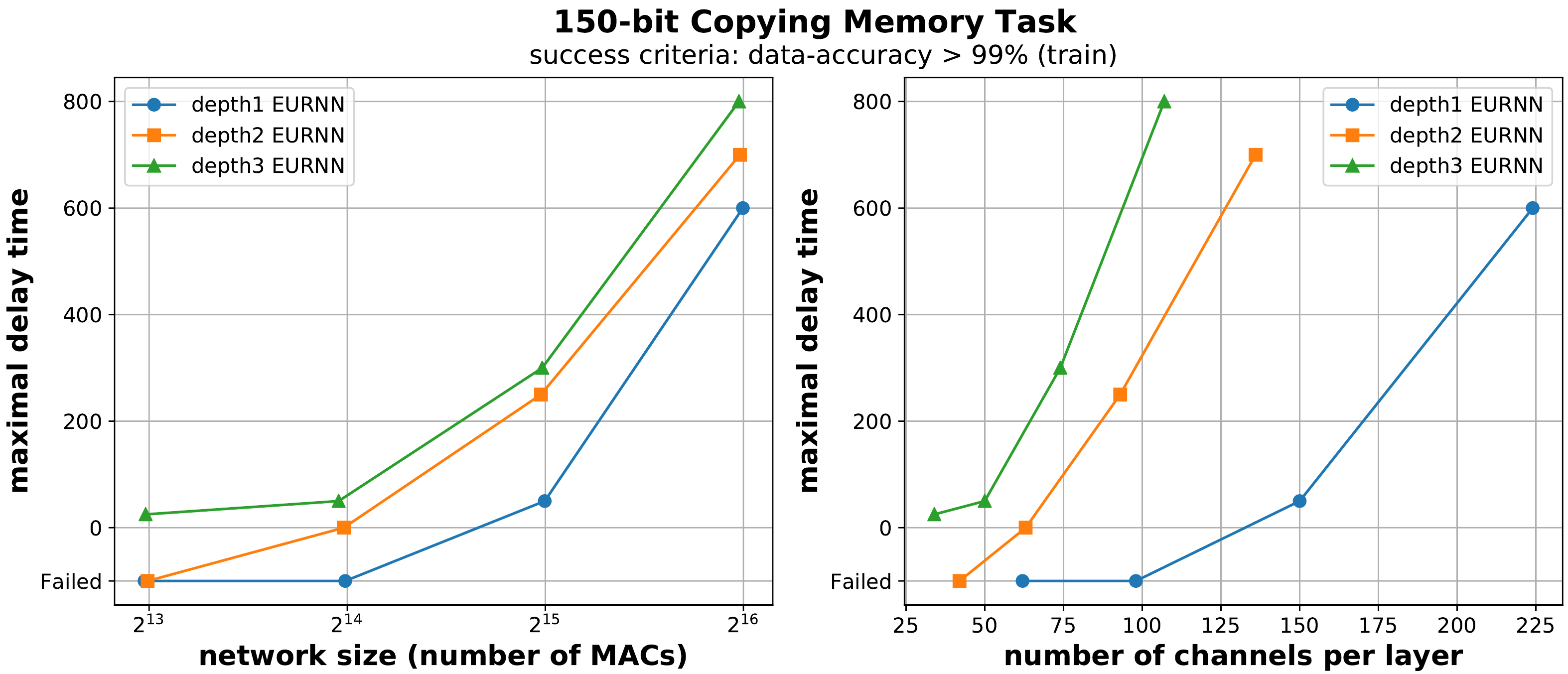}
		\caption{\label{fig:copy}
			Results of the Copying Memory Task, as defined in Section~\ref{sec:exp:copy}. The results are shown for networks of depths $1,2,3$ and sizes $2^{13}$~-~$2^{16}$ (measured in MACs). 
			We define success in the Copying Memory Task as achieving a data-accuracy $\geq 99\%$, \ie~being able to reproduce each character of the copied data with $99\%$ accuracy after a given delay time. For each network architecture, the plots report the longest delay time for which the architecture has been successful on the training set (test set results were very similar) as a function of network size (left) and number of channels per layer (right). We tested the performance on delay times up to $1000$, sampling delay times of $0,25,50,100,150,200,250,300$ and then in intervals of $100$. If a network cannot even solve the task for zero delay time, we mark it as ``Failed''. The advantage of deepening the network is evident, as for each tested network size, the recurrent network of depth $3$ outperforms the recurrent network of depth $2$, which outperforms the recurrent network of depth $1$. For a case of limited amount of resources \wrt~the task hardness, which occurs in the smaller network sizes, shallower networks fail to perform the task altogether (cannot reproduce the given sequence even after zero delay time), where deeper networks succeed. The displayed results clearly highlight the augmenting contribution of depth to the recurrent network's long-term memory capacity.} 
	\end{figure}	
	We present the results for this task in
	Figure~\ref{fig:copy}, where we compare the performance
	for networks of depths $1,2,3$ and of size in the range
	of $2^{13}$~-~$2^{16}$, measured in the number of 
	multiply-accumulate operations~(MACs). 
	Our measure of performance in the Copy Memory Task is referred to as the \emph{data-accuracy}, calculated as $\frac{1}{N}\frac{1}{m}\sum_{j=1}^{N}\sum_{t=m+B+1}^{T}\mathbf{1}[\hat{O}_t^j=O_t^j]$, where $N$ is the sample size, $O_t^j$ the correct output character at time $t$ for example $j$, and $\hat{O}_t^j=\argmax_{i\in[n+2]}{\y_t^j}$ the predicted character. The data-accuracy effectively reflects the per-character data reproduction ability, therefore it is defined only over the final $m$ time-steps when the memorized data is to be reproduced.
	In Figure~\ref{fig:copy}, we display for
	each network the longest delay time for which it is able
	to solve the task, demonstrating a clear advantage of depth in
	this task. 
	We measure the size of the network using MACs
	due to the fact that while orthogonal matrices have an 
	effective smaller number of parameters, EURNN still
	require the same number of MACs at inference time, hence
	it is a better representation of the resources they
	demand. Clearly, given an amount of resources, it is
	advantageous to allocate them in a stacked layer fashion
	for this long-term memory based task.
	
	\subsection{Start-End Similarity Task}\label{sec:exp:sim}
	The Start-End Similarity Task directly tests the
	recurrent network's expressive ability to integrate
	between the two halves of the input sequence. In this
	task, the network needs to determine how similar the two halves are. 
	The input is a sequence of $T$ characters
	$\{x_t\}_{t=1}^T$ from an alphabet
	$\left\{ a_{i}\right\} _{i=1}^{n}$, where the first
	$\nicefrac{T}{2}$ characters are denoted by `Start' and
	the rest by `End', similarly to previous sections. 
	Considering pairs of characters in the same
	relative position in `Start' and `End', \ie~the pairs
	$\left(x_{t},x_{t+\nicefrac{T}{2}}\right)$, we divide
	each input sequence into one of the following classes: 
	
	\begin{itemize}
		\item \textbf{1-similar}: `Start' and `End' are
		exactly the same $\nicefrac{T}{2}$ length string.
		\item \textbf{0.5-similar}: `Start' and `End' have
		exactly $\nicefrac{T}{4}$ matching pairs of
		characters (a randomly positioned half of the
		string is identical, and the other half is not).
		\item \textbf{0-similar}: no pair of characters
		$\left(x_{t},x_{t+\nicefrac{T}{2}}\right)$ match. 
	\end{itemize}
	
	\begin{figure}
		\centering
		\includegraphics[width=\linewidth]{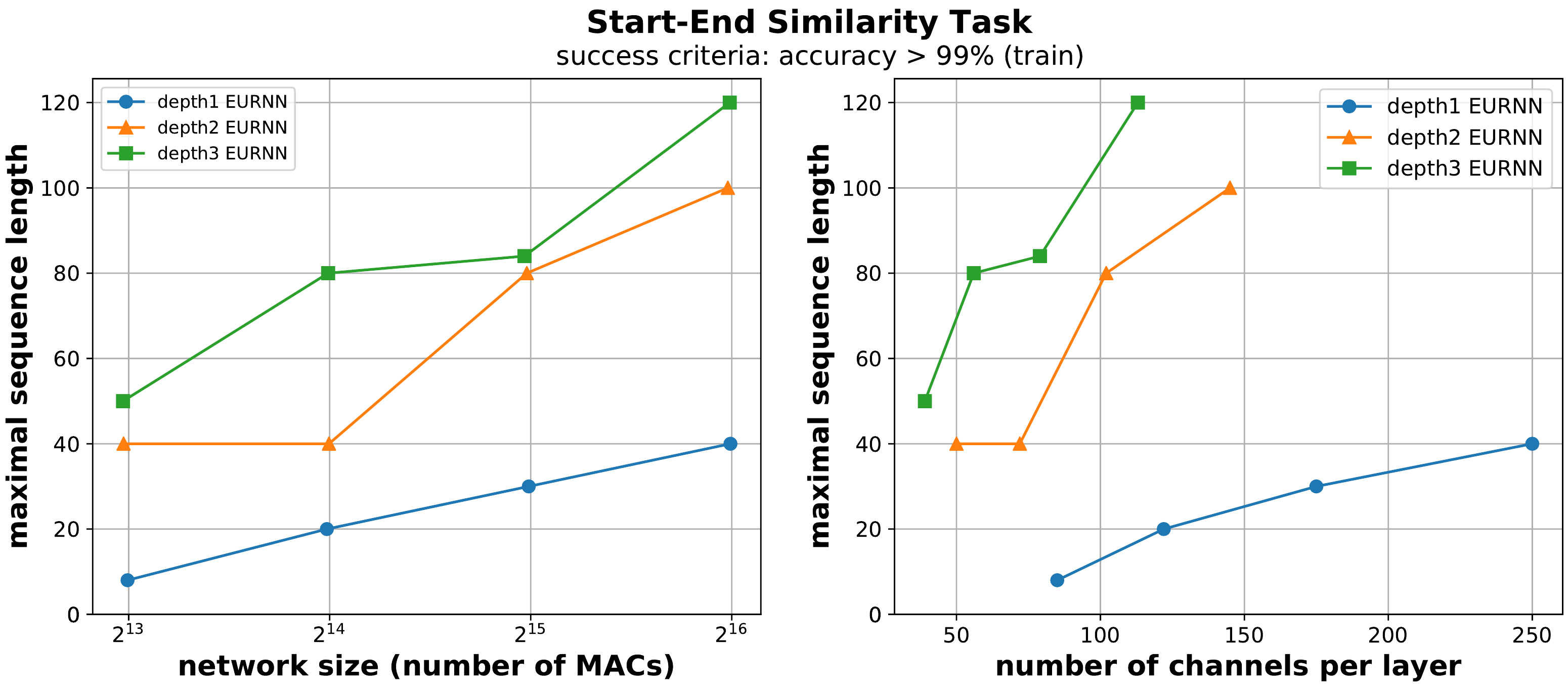}
		\caption{Results of the Start-End Similarity Task, as defined in Section~\ref{sec:exp:sim}. The results are shown for networks of depths $1,2,3$ and sizes $2^{13}$~-~$2^{16}$ (measured in MACs). We define success on the Start-End similarity task as training accuracy $\geq 99\%$ (test set results were very similar). For each network architecture, the plots report the longest input sequence length for which the architecture has been successful as a function of network size (left) and number of channels per layer (right). We tested the performance on sequences of lengths up to 140 characters in intervals of 10. It can be seen that for every given network size, a deeper network can model long-term dependencies more successfully than a shallower one. For example, a depth-$3$ network succeeds at solving the Start-End Similarity task for $T=120$ while a depth-$1$ network succeeds only for $T=40$.\label{fig:sim}}
	\end{figure}
	The task we examine is a classification task of a dataset 
	distributed uniformly over these three classes. Here, the
	recurrent networks are to produce a meaningful output
	only in the last time-step, determining in which class
	the input was, \ie~how similar the beginning of the input
	sequence is to its end.	
	Figure~\ref{fig:sim} shows the performance for networks of depths $1,2,3$ and sizes $2^{13}$~-~$2^{16}$, measured in MACs as explained above, on the Start-End Similarity Task. The clear advantage of depth is portrayed in this task as well, empirically demonstrating the enhanced ability of deep recurrent networks to model long-term elaborate dependencies in the input string.
	
	Overall, the empirical results presented in this section reflect well our theoretical findings, presented in Section~\ref{sec:results}.
	
	\ifdefined\SQUEEZE\vspace{-2mm}\fi
	\section{Discussion} \label{sec:discussion}
	\ifdefined\SQUEEZE\vspace{-2mm}\fi
	
	The notion of depth efficiency, by which deep networks efficiently express functions that would require shallow networks to have a super-linear size, is well established in the context of convolutional networks.
	However, recurrent networks differ from convolutional networks, as they are suited by design to tackle inputs of varying lengths.
	Accordingly, depth efficiency alone does not account for the remarkable performance of deep recurrent networks on long input sequences.
	In this paper, we identified a fundamental need for a quantifier of `time-series expressivity', quantifying the memory capacity of recurrent networks, which can account for the empirically undisputed advantage of depth in hard sequential tasks.
	In order to meet this need, we proposed a measure of the ability of recurrent networks to model long-term temporal dependencies, in the form of the Start-End separation rank.
	The separation rank was used to quantify dependencies in convolutional networks, and has roots in the field of quantum physics.
	The proposed Start-End separation rank measure adjusts itself to the temporal extent of the input series, and quantifies the ability of the recurrent network to correlate the incoming sequential data as time progresses.
	
	We analyzed the class of Recurrent Arithmetic Circuits, which are closely related to successful RNN architectures, and proved that the Start-End separation rank of deep RACs increases combinatorially with the number of channels and as the input sequence extends, while that of shallow RACs increases linearly with the number of channels and is independent of the input sequence length.
	These results, which demonstrate that depth brings forth an overwhelming advantage in the ability of recurrent networks to model long-term dependencies, were achieved by combining tools from the fields of measure theory, tensorial analysis, combinatorics, graph theory and quantum physics. 
	The above presented empirical evaluations support our theoretical findings, and provide a demonstration of their relevance for commonly used classes of recurrent networks.
	
	Such analyses may be readily extended to other architectural features employed in modern recurrent networks.
	Indeed, the same time-series expressivity question may now be applied to the different variants of LSTM networks, and the proposed notion of Start-End separation rank may be employed for quantifying their memory capacity.
	We have demonstrated that such a treatment can go beyond unveiling the origins of the success of a certain architectural choice, and leads to new insights.
	The above established observation that dependencies achievable by vanilla shallow recurrent network do not adapt at all to the sequence length, is an exemplar of this potential.
	
	Moreover, practical recipes may emerge by such theoretical analyses.
	The experiments preformed in \citep{hermans2013training}, suggest that shallow layers of recurrent networks are related to short time-scales, \eg~in speech: phonemes, syllables, words, while deeper layers appear to support dependencies of longer time-scales, \eg~full sentences, elaborate questions.
	These findings open the door to further depth related investigations in recurrent networks, and specifically the role of each layer in modeling temporal dependencies may be better understood.
	\citep{levine2018deep} establish theoretical observations which translate into practical conclusions regarding the number of hidden channels to be chosen for each layer in a deep convolutional network.
	The conjecture presented in this paper, by which the Start-End separation rank of recurrent networks grows combinatorially with depth, can similarly entail practical recipes for enhancing their memory capacity. Such analyses can lead to a profound understanding of the contribution of deep layers to the recurrent network's memory. Indeed, we view this work as an important step towards novel methods of matching the recurrent network architecture to the temporal dependencies in a given sequential dataset.
	
    	\section*{Acknowledgments}
	We acknowledge useful discussions with Nati Linial and Noam Weis, as well as the contribution of Nadav Cohen who
	originally suggested the connection between shallow Recurrent Arithmetic Circuits and the Tensor
	Train decomposition. 
	\section*{Funding}	
	This work is supported by the European Research Council (TheoryDL project)
	and by ISF Center grant 1790/12. 
	Y.L. is supported by the Adams
	Fellowship Program of the Israel Academy of Sciences
	and Humanities. 
    
	\section*{References}
	\bibliographystyle{imaiai}
	\bibliography{refs}

\ifx\undefined\BySame
\newcommand{\BySame}{\leavevmode\rule[.5ex]{3em}{.5pt}\ }
\fi
\ifx\undefined\textsc
\newcommand{\textsc}[1]{{\sc #1}}
\newcommand{\emph}[1]{{\em #1\/}}
\let\tmpsmall\small
\renewcommand{\small}{\tmpsmall\sc}
\fi
\begin{thebibliography}{99}

\bibitem{amini2012low}
\textsc{Amini, A., Karbasi, A.  {\small \&} Marvasti, F.}  (2012) Low-Rank
  Matrix Approximation Using Point-Wise Operators. \emph{IEEE Transactions on
  Information Theory}, \textbf{58}(1), 302--310.

\bibitem{amodei2016deep}
\textsc{Amodei, D., Ananthanarayanan, S., Anubhai, R., Bai, J., Battenberg, E.,
  Case, C., Casper, J., Catanzaro, B., Cheng, Q., Chen, G.  et~al.}  (2016)
  Deep speech 2: End-to-end speech recognition in english and mandarin. in
  \emph{International Conference on Machine Learning}, pp. 173--182.

\bibitem{arjovsky2016unitary}
\textsc{Arjovsky, M., Shah, A.  {\small \&} Bengio, Y.}  (2016) Unitary
  evolution recurrent neural networks. in \emph{International Conference on
  Machine Learning}, pp. 1120--1128.

\bibitem{bahdanau2014neural}
\textsc{Bahdanau, D., Cho, K.  {\small \&} Bengio, Y.}  (2014) Neural machine
  translation by jointly learning to align and translate. \emph{arXiv preprint
  arXiv:1409.0473}.

\bibitem{beylkin2009multivariate}
\textsc{Beylkin, G., Garcke, J.  {\small \&} Mohlenkamp, M.~J.}  (2009)
  Multivariate regression and machine learning with sums of separable
  functions. \emph{SIAM Journal on Scientific Computing}, \textbf{31}(3),
  1840--1857.

\bibitem{beylkin2002numerical}
\textsc{Beylkin, G.  {\small \&} Mohlenkamp, M.~J.}  (2002) Numerical operator
  calculus in higher dimensions. \emph{Proceedings of the National Academy of
  Sciences}, \textbf{99}(16), 10246--10251.

\bibitem{cho2014learning}
\textsc{Cho, K., Van~Merri{\"e}nboer, B., Gulcehre, C., Bahdanau, D., Bougares,
  F., Schwenk, H.  {\small \&} Bengio, Y.}  (2014) Learning phrase
  representations using RNN encoder-decoder for statistical machine
  translation. \emph{arXiv preprint arXiv:1406.1078}.

\bibitem{cirecsan2010deep}
\textsc{Cire{\c{s}}an, D.~C., Meier, U., Gambardella, L.~M.  {\small \&}
  Schmidhuber, J.}  (2010) Deep, big, simple neural nets for handwritten digit
  recognition. \emph{Neural computation}, \textbf{22}(12), 3207--3220.

\bibitem{cohen2016expressive}
\textsc{Cohen, N., Sharir, O.  {\small \&} Shashua, A.}  (2016) On the
  Expressive Power of Deep Learning: A Tensor Analysis. \emph{Conference On
  Learning Theory (COLT)}.

\bibitem{cohen2016convolutional}
\textsc{Cohen, N.  {\small \&} Shashua, A.}  (2016) Convolutional Rectifier
  Networks as Generalized Tensor Decompositions. \emph{International Conference
  on Machine Learning (ICML)}.

\bibitem{cohen2017inductive}
\textsc{\BySame{}}  (2017) Inductive bias of deep convolutional networks
  through pooling geometry. in \emph{5th International Conference on Learning
  Representations (ICLR)}.

\bibitem{cohen2017boosting}
\textsc{Cohen, N., Tamari, R.  {\small \&} Shashua, A.}  (2017) Boosting
  Dilated Convolutional Networks with Mixed Tensor Decompositions. \emph{arXiv
  preprint arXiv:1703.06846}.

\bibitem{NIPS2011_4350}
\textsc{Delalleau, O.  {\small \&} Bengio, Y.}  (2011) Shallow vs. Deep
  Sum-Product Networks. in \emph{Advances in Neural Information Processing
  Systems 24}, ed. by J.~Shawe-Taylor, R.~S. Zemel, P.~L. Bartlett, F.~Pereira,
   {\small \&} K.~Q. Weinberger, pp. 666--674. Curran Associates, Inc.

\bibitem{eldan2016power}
\textsc{Eldan, R.  {\small \&} Shamir, O.}  (2016) The power of depth for
  feedforward neural networks. in \emph{Conference on Learning Theory}, pp.
  907--940.

\bibitem{gers2000recurrent}
\textsc{Gers, F.~A.  {\small \&} Schmidhuber, J.}  (2000) Recurrent nets that
  time and count. in \emph{Neural Networks, 2000. IJCNN 2000, Proceedings of
  the IEEE-INNS-ENNS International Joint Conference on}, vol.~3, pp. 189--194.
  IEEE.

\bibitem{graves2013generating}
\textsc{Graves, A.}  (2013) Generating sequences with recurrent neural
  networks. \emph{arXiv preprint arXiv:1308.0850}.

\bibitem{graves2009novel}
\textsc{Graves, A., Liwicki, M., Fern{\'a}ndez, S., Bertolami, R., Bunke, H.
  {\small \&} Schmidhuber, J.}  (2009) A novel connectionist system for
  unconstrained handwriting recognition. \emph{IEEE transactions on pattern
  analysis and machine intelligence}, \textbf{31}(5), 855--868.

\bibitem{graves2013speech}
\textsc{Graves, A., Mohamed, A.-r.  {\small \&} Hinton, G.}  (2013) Speech
  recognition with deep recurrent neural networks. in \emph{Acoustics, speech
  and signal processing (icassp), 2013 ieee international conference on}, pp.
  6645--6649. IEEE.

\bibitem{hackbusch2006efficient}
\textsc{Hackbusch, W.}  (2006) On the efficient evaluation of coalescence
  integrals in population balance models. \emph{Computing}, \textbf{78}(2),
  145--159.

\bibitem{hackbusch2012tensor}
\textsc{\BySame{}}  (2012) \emph{Tensor spaces and numerical tensor calculus},
  vol.~42. Springer Science \& Business Media.

\bibitem{hardy1952inequalities}
\textsc{Hardy, G.~H., Littlewood, J.~E.  {\small \&} P{\'o}lya, G.}  (1952)
  \emph{Inequalities}. Cambridge university press.

\bibitem{harrison2003multiresolution}
\textsc{Harrison, R.~J., Fann, G.~I., Yanai, T.  {\small \&} Beylkin, G.}
  (2003) Multiresolution quantum chemistry in multiwavelet bases. in
  \emph{Computational Science-ICCS 2003}, pp. 103--110. Springer.

\bibitem{hermans2013training}
\textsc{Hermans, M.  {\small \&} Schrauwen, B.}  (2013) Training and analysing
  deep recurrent neural networks. in \emph{Advances in Neural Information
  Processing Systems}, pp. 190--198.

\bibitem{hochreiter1997long}
\textsc{Hochreiter, S.  {\small \&} Schmidhuber, J.}  (1997) Long short-term
  memory. \emph{Neural computation}, \textbf{9}(8), 1735--1780.

\bibitem{jing2016tunable}
\textsc{Jing, L., Shen, Y., Dub{\v{c}}ek, T., Peurifoy, J., Skirlo, S., LeCun,
  Y., Tegmark, M.  {\small \&} Solja{\v{c}}i{\'c}, M.}  (2016) Tunable
  efficient unitary neural networks (EUNN) and their application to RNNs.
  \emph{arXiv preprint arXiv:1612.05231}.

\bibitem{khrulkov2018expressive}
\textsc{Khrulkov, V., Novikov, A.  {\small \&} Oseledets, I.}  (2018)
  Expressive power of recurrent neural networks. in \emph{6th International
  Conference on Learning Representations (ICLR)}.

\bibitem{lecun1998mnist}
\textsc{LeCun, Y., Cortes, C.  {\small \&} Burges, C.~J.}  (1998) The MNIST
  database of handwritten digits. .

\bibitem{levine2018deep}
\textsc{Levine, Y., Yakira, D., Cohen, N.  {\small \&} Shashua, A.}  (2018)
  Deep Learning and Quantum Entanglement: Fundamental Connections with
  Implications to Network Design. in \emph{6th International Conference on
  Learning Representations (ICLR)}.

\bibitem{martens2011learning}
\textsc{Martens, J.  {\small \&} Sutskever, I.}  (2011) Learning recurrent
  neural networks with hessian-free optimization. in \emph{Proceedings of the
  28th International Conference on Machine Learning (ICML-11)}, pp. 1033--1040.
  Citeseer.

\bibitem{mohamed2012acoustic}
\textsc{Mohamed, A.-r., Dahl, G.~E.  {\small \&} Hinton, G.}  (2012) Acoustic
  modeling using deep belief networks. \emph{IEEE Transactions on Audio,
  Speech, and Language Processing}, \textbf{20}(1), 14--22.

\bibitem{orus2014practical}
\textsc{Or{\'u}s, R.}  (2014) A practical introduction to tensor networks:
  Matrix product states and projected entangled pair states. \emph{Annals of
  Physics}, \textbf{349}, 117--158.

\bibitem{oseledets2011tensor}
\textsc{Oseledets, I.~V.}  (2011) Tensor-train decomposition. \emph{SIAM
  Journal on Scientific Computing}, \textbf{33}(5), 2295--2317.

\bibitem{pascanu2013difficulty}
\textsc{Pascanu, R., Mikolov, T.  {\small \&} Bengio, Y.}  (2013) On the
  difficulty of training recurrent neural networks. in \emph{International
  Conference on Machine Learning}, pp. 1310--1318.

\bibitem{poon2011sum}
\textsc{Poon, H.  {\small \&} Domingos, P.}  (2011) Sum-product networks: A new
  deep architecture. in \emph{Computer Vision Workshops (ICCV Workshops), 2011
  IEEE International Conference on}, pp. 689--690. IEEE.

\bibitem{schmidhuber1992learning}
\textsc{Schmidhuber, J.~H.}  (1992) Learning complex, extended sequences using
  the principle of history compression.. \emph{Neural Computation}.

\bibitem{sharir2018expressive}
\textsc{Sharir, O.  {\small \&} Shashua, A.}  (2018) On the Expressive Power of
  Overlapping Architectures of Deep Learning. in \emph{6th International
  Conference on Learning Representations (ICLR)}.

\bibitem{sharirtractable}
\textsc{Sharir, O., Tamari, R., Cohen, N.  {\small \&} Shashua, A.}  (2016)
  Tractable Generative Convolutional Arithmetic Circuits. .

\bibitem{NIPS2016_6211}
\textsc{Stoudenmire, E.  {\small \&} Schwab, D.~J.}  (2016) Supervised Learning
  with Tensor Networks. in \emph{Advances in Neural Information Processing
  Systems 29}, ed. by D.~D. Lee, M.~Sugiyama, U.~V. Luxburg, I.~Guyon,  {\small
  \&} R.~Garnett, pp. 4799--4807. Curran Associates, Inc.

\bibitem{sutskever2011generating}
\textsc{Sutskever, I., Martens, J.  {\small \&} Hinton, G.~E.}  (2011)
  Generating text with recurrent neural networks. in \emph{Proceedings of the
  28th International Conference on Machine Learning (ICML-11)}, pp. 1017--1024.

\bibitem{telgarsky2015representation}
\textsc{Telgarsky, M.}  (2015) Representation benefits of deep feedforward
  networks. \emph{arXiv preprint arXiv:1509.08101}.

\bibitem{tieleman2012lecture}
\textsc{Tieleman, T.  {\small \&} Hinton, G.}  (2012) Lecture 6.5-rmsprop:
  Divide the gradient by a running average of its recent magnitude.
  \emph{COURSERA: Neural networks for machine learning}, \textbf{4}(2), 26--31.

\bibitem{wisdom2016full}
\textsc{Wisdom, S., Powers, T., Hershey, J., Le~Roux, J.  {\small \&} Atlas,
  L.}  (2016) Full-capacity unitary recurrent neural networks. in
  \emph{Advances in Neural Information Processing Systems}, pp. 4880--4888.

\bibitem{wu2016multiplicative}
\textsc{Wu, Y., Zhang, S., Zhang, Y., Bengio, Y.  {\small \&} Salakhutdinov,
  R.~R.}  (2016) On multiplicative integration with recurrent neural networks.
  in \emph{Advances in Neural Information Processing Systems}, pp. 2856--2864.

\bibitem{zhang2016architectural}
\textsc{Zhang, S., Wu, Y., Che, T., Lin, Z., Memisevic, R., Salakhutdinov,
  R.~R.  {\small \&} Bengio, Y.}  (2016) Architectural complexity measures of
  recurrent neural networks. in \emph{Advances in Neural Information Processing
  Systems}, pp. 1822--1830.

\end{thebibliography}
	
	\clearpage
	\appendix
	
	\section{Tensor Network Representation of Recurrent Arithmetic circuits} \label{app:rac_tns}
	
	In this section, we expand our algebraic view on recurrent networks and make use
	of a graphical approach to tensor decompositions referred to as Tensor Networks
	(TNs). The tool of TNs is mainly used in the many-body quantum physics literature
	for a graphical decomposition of tensors, and has been recently connected to the
	deep learning field by~\citep{levine2018deep}, who constructed a deep
	convolutional network in terms of a TN. The use of TNs in machine learning has
	appeared in an empirical context, where \citep{NIPS2016_6211} trained a Matrix
	Product State (MPS) TN to preform supervised learning tasks on the MNIST
	dataset~\citep{lecun1998mnist}. The constructions presented in this section
	suggest a separation in expressiveness between recurrent networks
	of different depths, as formulated by Conjecture~\ref{conjecture:high_L}.
	
	We begin in Appendix~\ref{app:rac_tns:tns_intro} by providing a brief introduction to TNs.
	Next, we present in Appendix~\ref{app:rac_tns:shallow} the TN which corresponds to the calculation of a shallow RAC, and tie it to a common TN architecture referred to as a \emph{Matrix Product State} (MPS) (see overview in e.g. \citep{orus2014practical}), and equivalently to the \emph{tensor train} (TT) decomposition \citep{oseledets2011tensor}.
	Subsequently, we present in Appendix~\ref{app:rac_tns:deep} a TN construction of a deep RAC, and emphasize the characteristics of this construction that are the origin of the enhanced ability of deep RACs to model elaborate temporal dependencies.
	Finally, in Appendix~\ref{app:rac_tns:conjecture}, we make use of the above TNs construction in order to formally motivate Conjecture~\ref{conjecture:high_L}, according to which the Start-End separation rank of RACs grows combinatorially with depth.
	
	\subsection{Introduction to Tensor Networks}\label{app:rac_tns:tns_intro}
	
	A TN is a weighted graph, where each node corresponds to a tensor whose order is
	equal to the degree of the node in the graph. Accordingly, the edges emanating
	out of a node, also referred to as its legs, represent the different modes of
	the corresponding tensor. The weight of each edge in the graph, also referred to
	as its bond dimension, is equal to the dimension of the appropriate tensor mode.
	In accordance with the relation between mode, dimension and index of a tensor
	presented in Section~\ref{sec:corr:grid}, each edge in a TN is represented by
	an index that runs between $1$ and its bond dimension.
	Figure~\ref{fig:tns_intro}a shows three examples:
	(1) A vector, which is a tensor of order $1$, is represented by a node with one
	leg. (2) A matrix, which is a tensor of order $2$, is represented by a node with
	two legs. (3) Accordingly, a tensor of order $N$ is represented in the TN as a
	node with $N$ legs.
	
	\begin{figure}
		\centering
		\includegraphics[width=\linewidth]{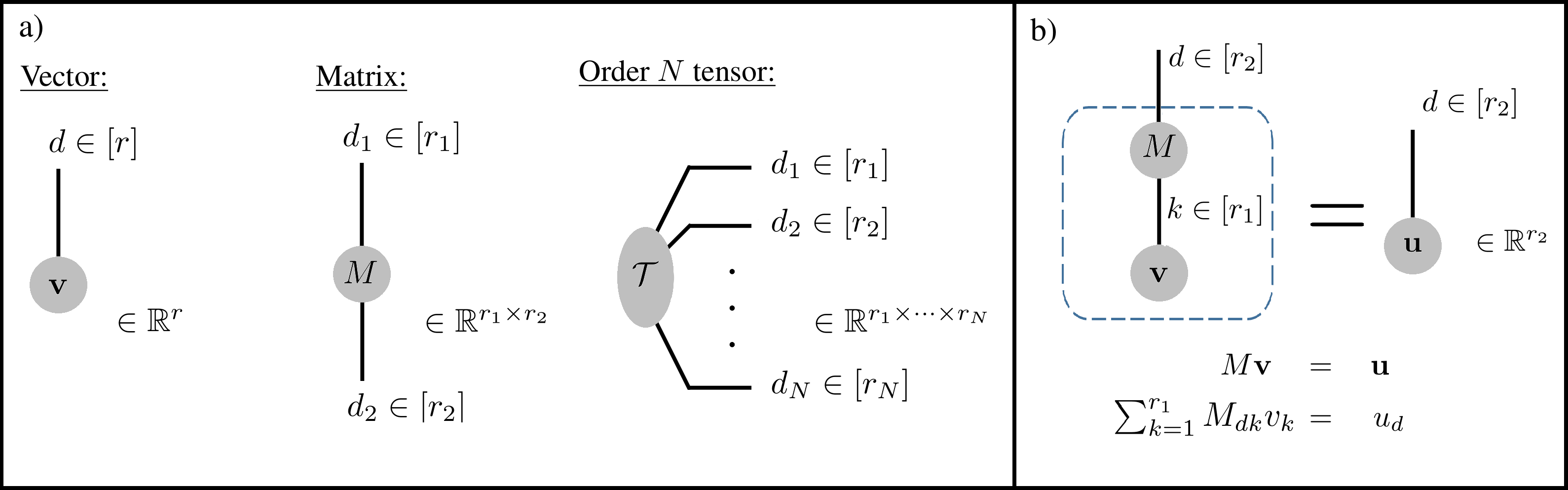}
		\caption{
			A quick introduction to Tensor Networks (TNs). a) Tensors in the TN are
			represented by nodes. The degree of the node corresponds to the order of
			the tensor represented by it. b) A matrix multiplying a vector in TN
			notation. The contracted index $k$, which connects two nodes, is summed
			upon, while the open index $d$ is not. The number of open indices equals
			the order of the tensor represented by the entire network. All of the
			indices receive values that range between $1$ and their bond dimension.
			The contraction is marked by the dashed line.
		}
		\label{fig:tns_intro}
	\end{figure}
	
	We move on to present the connectivity properties of a TN. Edges which connect
	two nodes in the TN represent an operation between the two corresponding
	tensors. A index which represents such an edge is called a contracted index, and
	the operation of contracting that index is in fact a summation over all of the
	values it can take. An index representing an edge with one loose end is called
	an open index. The tensor represented by the entire TN, whose order is equal to
	the number of open indices, can be calculated by summing over all of the
	contracted indices in the network. An example for a contraction of a simple TN
	is depicted in Figure~\ref{fig:tns_intro}b. There, a
	TN corresponding to the operation of multiplying a vector $\vv \in \R^{r_1}$ by
	a matrix $M\in \R^{r_2 \times r_1}$ is performed by summing over the only
	contracted index, $k$. As there is only one open index, $d$, the result of
	contracting the network is an order $1$ tensor (a vector): $\uu \in \R^{r_2}$
	which upholds $\uu = M\vv$. Though we use below the contraction of indices in
	more elaborate TNs, this operation can be essentially viewed as a generalization
	of matrix multiplication.
	
	\subsection{Shallow RAC Tensor Network}
	\label{app:rac_tns:shallow}
	
	The computation of the output at time $T$ that is preformed by the shallow
	recurrent network given by Equations~\eqref{eq:shallow_rn} and~\eqref{eq:g_rac}, or
	alternatively by Equations~\eqref{eq:score} and~\eqref{eq:tt_decomp}, can be written in
	terms of a TN. Figure~\ref{fig:shallow_rac_tn}a
	shows this TN, which given some initial hidden state $\h_0$, is essentially a
	temporal concatenation of a unit cell that preforms a similar computation at
	every time-step, as depicted in
	Figure~\ref{fig:shallow_rac_tn}b. For any time
	$t<T$, this unit cell is composed of the input weights matrix, $W^\textrm{I}$,
	contracted with the inputs vector, $\ff(\x^t)$, and the hidden weights matrix,
	$W^\textrm{H}$, contracted with the hidden state vector of the previous time-step, $\h^{t-1}$. The final component in each unit cell is the $3$ legged
	triangle representing the order $3$ tensor $\delta\in\R^{R\times R \times R}$,
	referred to as the \emph{$\delta$ tensor}, defined by:
	\begin{equation}
	\begin{array}{c}
	\delta_{i_1i_2i_3}\equiv\left\{ \begin{array}{c}
	1,\quad ~~~~i_1=i_2=i_3\\
	0,\quad ~~~~~~otherwise
	\end{array}\right.,\\
	\end{array}\label{eq:deltadef}
	\end{equation}
	with $i_j\in[R]~\forall j\in[3]$, \ie~its entries are equal to $1$ only on the
	super-diagonal and are zero otherwise. The use of a triangular node in the TN is
	intended to remind the reader of the restriction given in Equation~\eqref{eq:deltadef}.
	The recursive relation that is defined by the unit cell, is given by the TN in
	Figure~\ref{fig:shallow_rac_tn}b:
	\begin{align}
	h^t_{k_t}=
	\sum_{k_{t-1},\tilde{k}_{t-1},\tilde{d}_{t}=1}^{R}
	\sum_{d_{t}=1}^{M}
	W^{\textrm{ H}}_{\tilde{k}_{t-1}k_{t-1}}h^{t-1}_{k_{t-1}}
	W^{\textrm{ I}}_{\tilde{d}_td_t}f_{d_t}(\x^t)
	\delta_{\tilde{k}_{t-1}\tilde{d}_tk_t}=
	\nonumber \\
	~~~~~~~~~~~
	\sum_{\tilde{k}_{t-1}\tilde{d}_t=1}^{R}
	(W^{\textrm{ H}}\h^{t-1})_{\tilde{k}_{t-1}}
	(W^{\textrm{ I}}\ff(\x^t))_{\tilde{d}_t}
	\delta_{\tilde{k}_{t-1}\tilde{d}_tk_t}
	=
	(W^{\textrm{ H}}\h^{t-1})_{k_t}
	(W^{\textrm{ I}}\ff(\x^t))_{k_t},
	\label{eq:buildingblock1}
	\end{align}
	where $k_t\in [R]$. In the first equality, we simply follow the TN prescription
	and write a summation over all of the contracted indices in the left hand side
	of Figure~\ref{fig:shallow_rac_tn}b, in the
	second equality we use the definition of matrix multiplication, and in the last
	equality we use the definition of the $\delta$ tensor. The component-wise
	equality of Equation~\eqref{eq:buildingblock1} readily implies
	$\h^t=(W^{\textrm{H}}\h^{t-1})\odot(W^{\textrm{ I}}\ff(\x^t))$, reproducing the
	recursive relation in Equations~\eqref{eq:shallow_rn} and~\eqref{eq:g_rac}, which defines
	the operation of the shallow RAC. From the above treatment, it is evident that
	the restricted $\delta$ tensor is in fact the component in the TN that yields
	the element-wise multiplication property. After $T$ repetitions of the unit cell
	calculation with the sequential input $\{\x^t\}_{t=1}^{T }$, a final
	multiplication of the hidden state vector $\h^T$ by the output weights matrix
	$W^\textrm{O}$ yields the output vector $\y^{T,1,\Theta}$.
	
	\begin{figure}
		\centering
		\includegraphics[width=\linewidth]{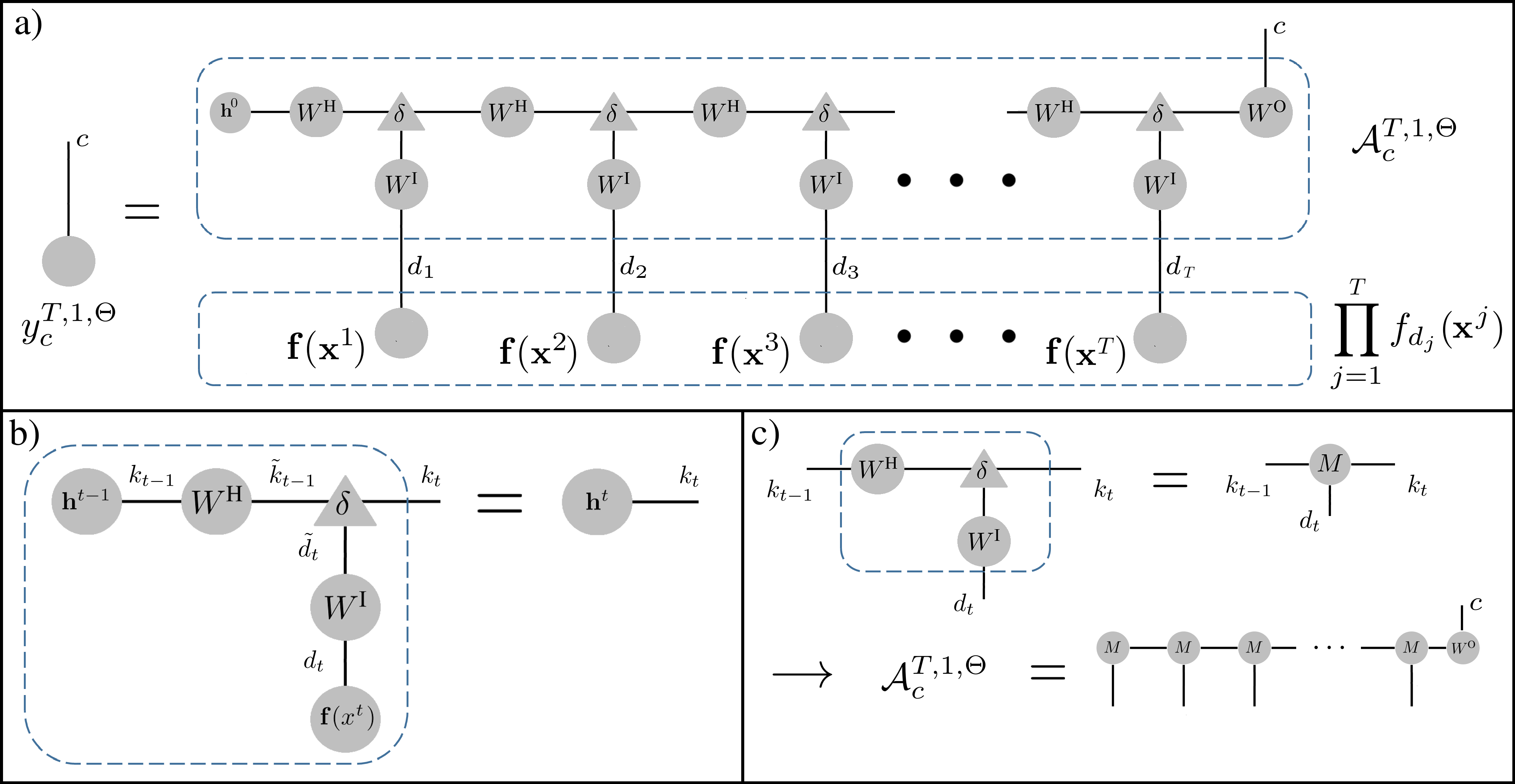}
		
		\caption{
			a) The Tensor Network representing the calculation performed by a shallow RAC.
			b) A Tensor Network construction of the recursive relation given an Equation~\eqref{eq:shallow_rn}.
			c) A presentation of the shallow RAC weights tensor in a standard MPS form.
		}
		\label{fig:shallow_rac_tn}
	\end{figure}
	
	The tensor network which represents the order $T$ shallow RAC weights tensor
	$\A^{T,1, \Theta}_c$, which appears in Equations~\eqref{eq:score}
	and~\eqref{eq:tt_decomp}, is given by the TN in the upper part of
	Figure~\ref{fig:shallow_rac_tn}a. In
	Figure~\ref{fig:shallow_rac_tn}c, we show that
	by a simple contraction of indices, the TN representing the shallow RAC weights
	tensor $\A^{T,1, \Theta}_c$ can be drawn in the form of a standard MPS TN. This
	TN allows the representation of an order $T$ tensor with a linear (in $T$)
	amount of parameters, rather than the regular exponential amount ($\A^{T,1, \Theta}_c$ has $M^T$
	entries). The decomposition which corresponds to this MPS TN is known as the Tensor
	Train (TT) decomposition of rank $R$ in the tensor analysis community, its
	explicit form given in Equation~\eqref{eq:tt_decomp}.
	
	The presentation of the shallow
	recurrent network in terms of a TN allows the employment of the min-cut
	analysis, which was introduced by \citep{levine2018deep} in the context of
	convolutional networks, for quantification of the information flow across time
	modeled by the shallow recurrent network. This was indeed preformed in our proof
	of the shallow case of Theorem~\ref{theorem:main_result} (see Appendix~\ref{app:proofs:main_result:shallow} for further details).
	We now move on to present the
	computation preformed by a deep recurrent network in the language of TNs.
	
	\subsection{Deep RAC Tensor Network}
	\label{app:rac_tns:deep}
	
	The construction of a TN which matches the calculation of a deep recurrent
	network is far less trivial than that of the shallow case, due to the seemingly
	innocent property of reusing information which lies at the heart of the
	calculation of deep recurrent networks. Specifically, all of the hidden states
	of the network are reused, since the state of each layer at every time-step is
	duplicated and sent as an input to the calculation of the same layer in the next
	time-step, and also as an input to the next layer up in the same time-step (see
	bottom of Figure~\ref{fig:recurrent_net}). The required
	operation of duplicating a vector and sending it to be part of two different
	calculations, which is simply achieved in any practical setting, is actually
	impossible to represent in the framework of TNs. We formulate this notion in the
	following claim:
	\begin{claim} \label{claim:no_clone}
		Let $v\in\R^P,P\in\N$ be a vector. $v$ is represented by a node with one leg
		in the TN notation. The operation of duplicating this node, \ie~ forming two
		separate nodes of degree $1$, each equal to $v$, cannot be achieved by any
		TN.
	\end{claim}
	\begin{proof}
		We assume by contradiction that there exists a Tensor Network $\phi$ which operates on any vector $v\in\R^P$ and clones it to two separate nodes of degree $1$, each equal to $v$, to form an overall TN representing $v\otimes v$.
		Component wise, this implies that $\phi$ upholds $\forall v\in\R^P:~\sum_{i=1}^P \phi_{ijk}v_i=v_jv_k$.
		By our assumption, $\phi$ duplicates the standard basis elements of $\R^P$, denoted $\{\hat{\e}^{(\alpha)}\}_{\alpha=1}^P$, meaning that $\forall \alpha\in[P]$:
		\begin{equation}\label{eq:no_clone}
		\sum_{i=1}^P \phi_{ijk}\hat{e}^{(\alpha)}_i=\hat{e}^{(\alpha)}_j\hat{e}^{(\alpha)}_k.
		\end{equation}
		By definition of the standard basis elements, the left hand side of Equation~\eqref{eq:no_clone} takes the form $\phi_{\alpha jk}$ while the right hand side equals $1$ only if $j=k=\alpha$, and otherwise $0$. Utilizing the $\delta$-tensor notation presented in Equation~\eqref{eq:deltadef}, in order to successfully clone the standard basis elements, Equation~\eqref{eq:no_clone} implies that $\phi$ must uphold $\phi_{\alpha jk}=\delta_{\alpha jk}$. However, for $v=\mathbf{1}$, \ie~$\forall j\in[P]:~v_j=1$, a cloning operation does not take place when using this value of $\phi$, since $\sum_{i=1}^P \phi_{ijk}v_i=\sum_{i=1}^P \delta_{ijk}=\delta_{jk}\neq 1= v_iv_j$, in contradiction to $\phi$ duplicating any vector in $\R^P$.
	\end{proof}
	
	Claim~\ref{claim:no_clone} seems to pose a hurdle in our pursuit of a TN
	representing a deep recurrent network. Nonetheless, a form of such a TN may be
	attained by a simple `trick'~--~in order to model the duplication that is
	inherently present in the deep recurrent network computation, we resort to
	duplicating the input data itself. By this technique, for every duplication that
	takes place along the calculation, the input is inserted into the TN multiple
	times, once for each sequence that leads to the duplication point. This
	principle, which allows us to circumvent the restriction imposed by
	Claim~\ref{claim:no_clone}, yields the elaborate TN construction of deep RACs
	depicted in Figure~\ref{fig:deep_rac_tn}.
	\begin{figure}
		\centering
		\includegraphics[width=\linewidth]{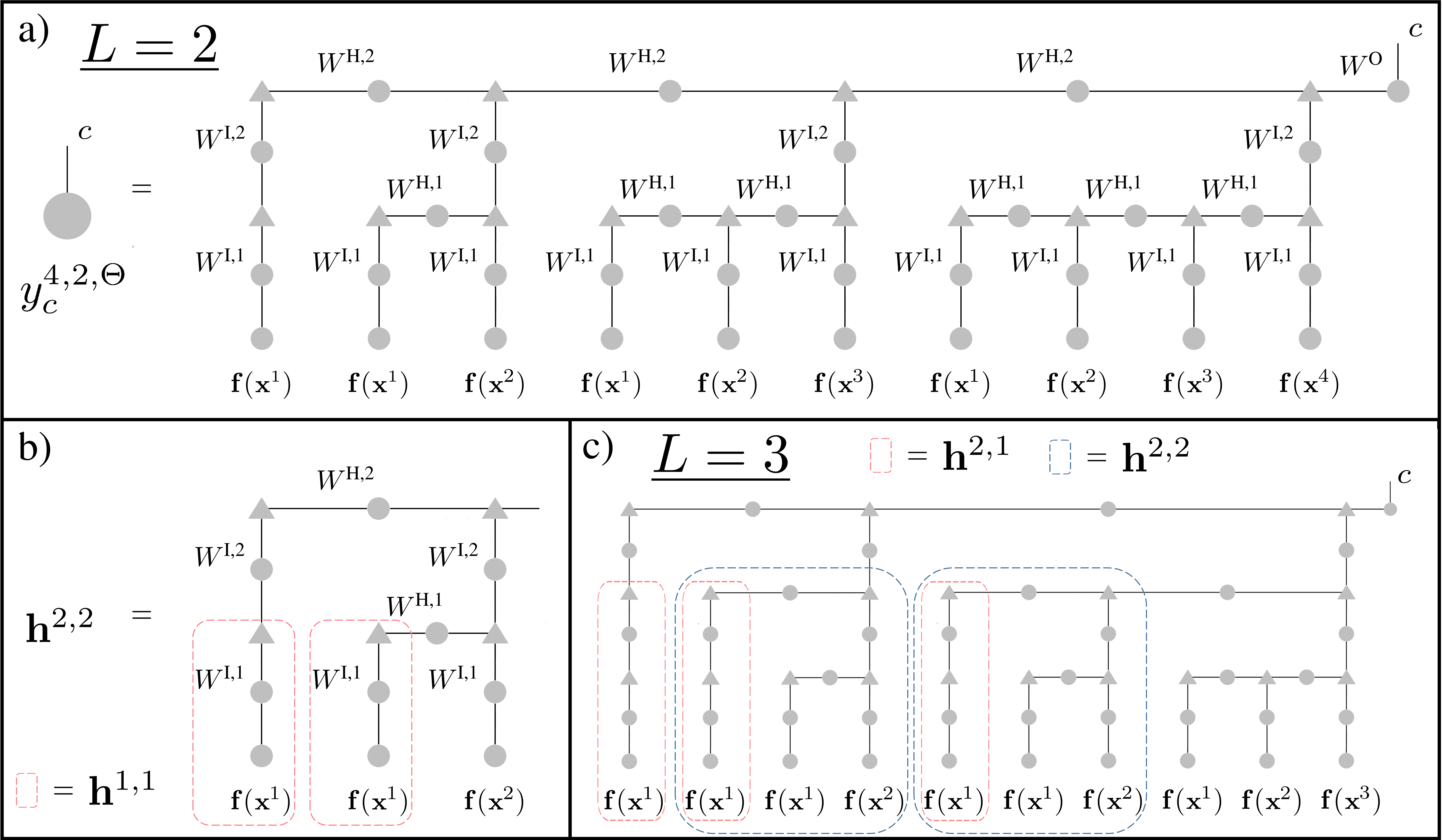}
		\caption{
			a) The Tensor Network representing the calculation preformed by a depth $L=2$ RAC after 4 time-steps.
			b) A Tensor Network construction of the hidden state $\h^{2,2}$ [see Equation~\eqref{eq:dup_example}], which involves duplication of the hidden state $ \h^{1,1}$ that is achieved by duplicating the input $x^1$.
			c) The Tensor Network representing the calculation preformed by a depth $L=3$ RAC after 3 time-steps. Here too, as in any deep RAC, several duplications take place.
		}
		\label{fig:deep_rac_tn}
	\end{figure}
	
	It is important to note that these TNs, which grow exponentially in size as the
	depth $L$ of the recurrent network represented by them increases, are merely a
	theoretical tool for analysis and not a suggested implementation scheme for deep
	recurrent networks. The actual deep recurrent network is constructed according to the
	simple scheme given at the bottom of
	Figure~\ref{fig:recurrent_net}, which grows only
	linearly in size as the depth $L$ increases, despite the corresponding TN
	growing exponentially. In fact, this exponential `blow-up' in the size of the
	TNs representing the deep recurrent networks is closely related to their ability
	to model more intricate dependencies over longer periods of time in comparison
	with their shallower counterparts, which was established in
	Section~\ref{sec:results}.
	
	Figure~\ref{fig:deep_rac_tn} shows TNs which correspond to depth $L=2,3$ RACs.
	Even though the TNs in Figure~\ref{fig:deep_rac_tn} seem rather convoluted and
	complex, their architecture follows clear recursive rules. In
	Figure~\ref{fig:deep_rac_tn}a, a depth $L=2$
	recurrent network is presented, spread out in time onto $T=4$ time-steps.
	To understand the logic underlying the input duplication process, which in turn
	entails duplication of entire segments of the TN, we focus on the calculation of
	the hidden state vector $\h^{2,2}$ that is presented in
	Figure~\ref{fig:deep_rac_tn}b. When the first
	inputs vector, $\ff(\x^1)$, is inserted into the network, it is multiplied by
	$W^{{\textrm I},1}$ and the outcome is equal to $\h^{1,1}$. Note that in this
	figure, the initial condition for each layer $l\in L$, $\h^{l,0}$, is chosen
	such that a vector of ones will be present in the initial element-wise
	multiplication: $(\h^{0,l})^T={\bf{1}}^T(W^{{\textrm H},l})^{\dagger}$,where $\dagger$ denotes the pseudoinverse operation.
	
	Next, $\h^{1,1}$ is used in two different places, as an inputs vector to layer $L=2$
	at time $t=1$, and as a hidden state vector in layer $L=1$ for time $t=2$
	calculation. Our input duplication technique inserts $\ff(\x^1)$ into the network
	twice, so that the same exact $\h^{1,1}$ is achieved twice in the TN, as marked
	by the red dotted line in
	Figure~\ref{fig:deep_rac_tn}b. This way, every copy
	of $\h^{1,1}$ goes to the appropriate segment of the calculation, and indeed the
	TN in Figure~\ref{fig:deep_rac_tn}b holds the
	correct value of $\h^{2,2}$:
	\begin{equation} \label{eq:dup_example}
	\h^{2,2} = \left(
	\vphantom{(W^{{\textrm H},1}\h^{1,1})\odot(W^{{\textrm I},1}\ff(\x^2))}
	W^{{\textrm H},2}W^{{\textrm I},2}\h^{1,1}\right)\odot \left(W^{{\textrm I},2}((W^{{\textrm H},1}\h^{1,1})\odot(W^{{\textrm I},1}\ff(\x^2)))\right).
	\end{equation}
	
	The extension to deeper layers leads us to a fractal structure of the TNs,
	involving many self similarities, as in the $L=3$ example given in
	Figure~\ref{fig:deep_rac_tn}c. The duplication of
	intermediate hidden states, marked in red and blue in this example, is the
	source of the apparent complexity of this $L=3$ RAC TN.
	Generalizing the above $L=1,2,3$ examples, a
	TN representing an RAC of general depth $L$ and of $T$ time-steps, would involve
	in its structure $T$ duplications of TNs representing RACs of depth $L-1$, each
	of which has a distinct length in time-steps $i$, where $i\in [T]$. This fractal structure leads to an increasing with
	depth complexity of the TN representing the depth $L$ RAC computation, which
	we show in the next subsection to motivate the combinatorial lower bound on the Start-End separation rank of deep RACs, given in Conjecture~\ref{conjecture:high_L}.
	
	\subsection{A Formal Motivation for Conjecture~\ref{conjecture:high_L}}
	\label{app:rac_tns:conjecture}
	
	The above presented construction of TNs which correspond to deep RACs, allows us to further investigate the effect of network depth on its ability to model long-term temporal dependencies. We present below a formal motivation for the lower bound on the Start-End separation rank of deep recurrent networks, given in Conjecture~\ref{conjecture:high_L}. Though our analysis employs TNs visualizations, it is formal nonetheless -- these graphs represent the computation in a well-defined manner (see Appendices~\ref{app:rac_tns:tns_intro}-\ref{app:rac_tns:deep} above).
	
	Our conjecture relies on the fact that it is sufficient
	to find a specific instance of the network parameters $\Theta\times\h^{0,l}$ for which
	$\mat{\A(y^{T,L, \Theta}_c)}_{S,E}$ achieves a certain rank, in order for this rank to
	be a lower bound on the Start-End separation rank of the network. This follows from combining Claim~\ref{claim:grid_sep_deep} and Lemma~\ref{lemma:poly_full_rank}. Claim~\ref{claim:grid_sep_deep} assures us that the Start-End separation rank of the function realized by an RAC of any depth $L$, is lower bounded by the rank of the matrix obtained by the corresponding grid tensor matricization: $\sep{S,E}{y^{T,L, \Theta}_c}\geq \rank{\mat{\A(y^{T,L, \Theta}_c)}_{S,E}}$. Thus, one must show that $\rank{\mat{\A(y^{T,L, \Theta}_c)}_{S,E}} \geq${\tiny{$\multiset{\min\{M,R\}}{\multiset{T/2}{L-1}}$}} for all of the values of parameters $\Theta\times\h^{0,l}$ but a set of Lebesgue measure zero, in order to establish the lower
	bound in Conjecture~\ref{conjecture:high_L}.
	Next, we rely on Lemma~\ref{lemma:poly_full_rank}, which states that since the entries of
	$\A(y^{T,L, \Theta}_c)$ are polynomials in the deep recurrent network's
	weights, it suffices to find a single example for which the rank of the
	matricized grid tensor is greater than the desired lower bound. Finding
	such an example would indeed imply that for almost all of the values of
	the network parameters, the desired inequality holds.

	\begin{figure}
		\centering
		\includegraphics[width=1\linewidth]{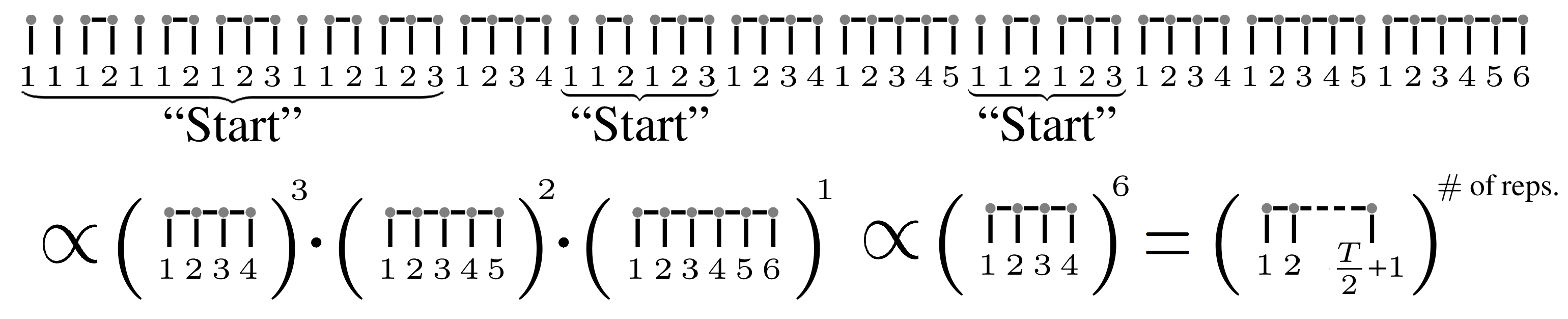}
		\caption{Above: TN representing the computation of a depth $L=3$ RAC after $T=6$ time-steps, when choosing $\WI{2}$ to be of rank-1. See full TN, for general values of the weight matrices, in Figure~\ref{fig:conjecture}. Below: Reduction of this TN to the factors affecting the Start-End matricization of the grid tensor represented by the TN.}
		\label{fig:conjecture_app}
	\end{figure}

	In the following, we choose a weight assignment that effectively `separates' between the first layer and higher layers, in the sense that $\WI{2}$ is of rank-1. This is done in similar spirit to the assignment used in the proof of Theorem~\ref{theorem:main_result}, in which $\WI{2}_{ij} \equiv \delta_{i1}$ (see Appendix~\ref{app:proofs:main_result}). Under this simplifying assignment, which suffices for our purposes according to the above discussion, the entire computation performed in deeper layers contributes only a constant factor to the matricized grid tensor. In this case, the example of the TN corresponding to an RAC of depth $L=3$ after $T=6$ time-steps, which is shown in full in Figure~\ref{fig:conjecture}, takes the form shown in the upper half of Figure~\ref{fig:conjecture_app}.
	Next, in order to evaluate $\rank{\mat{\A(y^{T,L, \Theta}_c)}_{S,E}}$, we note that graph segments which involve only indices from the ``Start'' set, will not affect the rank of the matrix under mild conditions on $\WI{1}, \WH{1}$ (for example, this holds if $\WI{1}$ is fully ranked and does not have vanishing elements, and $\WH{1}=I$). Specifically, under the Start-End matricization these segments will amount to a different constant multiplying each row of the matrix. For the example of the RAC of depth $L=3$ after $T=6$ time-steps, this amounts to the effective TN given in the bottom left side of Figure~\ref{fig:conjecture_app}.
	Finally, the dependence of this TN on the indices of time-steps $\{\nicefrac{T}{2}+2,\ldots ,T\}$, namely those outside of the basic unit involving indices of time-steps $\{1,\ldots ,\nicefrac{T}{2}+1\}$, may only increase the resulting Start-End matricization rank (this holds due to the temporal invariance of the recurrent network's weights). Thus, we are left with an effective TN resembling the one shown in Section~\ref{sec:results:conjecture}, where the basic unit separating ``Start" and ``End" indices is raised to the power of the number of its repetitions in the graph.
	In the following, we prove a claim according to which the number of repetitions of this basic unit in the TN graph increases combinatorially with the depth of the RAC:
	
	\begin{claim}
		Let $\phi(T,L,R)$ be the TN representing the computation performed after
		$T$ time-steps by an RAC with $L$ layers and $R$ hidden channels per layer. Then, the number of occurrences in layer $L=1$ of the basic unit connecting ``Start" and ``End" indices (bottom right in Figure~\ref{fig:conjecture_app}), is exactly {\tiny{\multiset{T/2}{L-1}}}.
	\end{claim}
	\begin{proof}
		Let $y^{T,L, \Theta}_c$ be the function computing the output after
		$T$ time-steps of an RAC with $L$ layers, $R$ hidden channels per layer and
		weights denoted by $\Theta$. In order to focus on repetitions in layer $L=1$, we assign $\WI{2}_{ij} \equiv \delta_{i1}$ for which the following upholds (see a similar and more detailed derivation in Appendix~\ref{app:proofs:main_result}):
		\begin{align*}
		\A(y^{T,L, \Theta}_c)_{d_1,\ldots,d_T}
		= \left(Const.\right)&\prod_{t_L=1}^T\prod_{t_{L-1}=1}^{t_L}\cdots\prod_{t_2=1}^{t_3} \sum_{r_1,\ldots,r_{t_2}=1}^R\left( \prod_{j=1}^{t_2}
		\WI{1}_{r_jd_j}\prod_{j=1}^{t_2-1}\WH{1}_{r_jr_{j+1}}\right)\\
		= \left(Const.\right)(V_{d_1\ldots d_{\nicefrac{T}{2}}})\prod_{t_L=\nicefrac{T}{2}+1}^T&\prod_{t_{L-1}=\nicefrac{T}{2}+1}^{t_L}\cdots\prod_{t_2=\nicefrac{T}{2}+1}^{t_3} \sum_{r_1,\ldots,r_{t_2}=1}^R\left( \prod_{j=1}^{t_2}
		\WI{1}_{r_jd_j}\prod_{j=1}^{t_2-1}\WH{1}_{r_jr_{j+1}}\right),
		\end{align*}
		where the constant term in the first line is the contribution of the deeper layers under this assignment, and the tensor $V_{d_1\ldots d_{\nicefrac{T}{2}}}$, which becomes a vector under the Start-End matricization, reflects the contribution of the ``Start'' set indices. Observing the argument of the chain of products in the above expression, $\sum_{r_1,\ldots,r_{t_2}=1}^R\left( \prod_{j=1}^{t_2}
		\WI{1}_{r_jd_j}\prod_{j=1}^{t_2-1}\WH{1}_{r_jr_{j+1}}\right)$, it is an order $t_2$ tensor, exactly given by the TN representing the computation of a depth $L=1$ RAC after $t_2$ time-steps.
		Specifically, for $t_2=\nicefrac{T}{2}+1$, it is exactly equal to the basic TN unit connecting ``Start" and ``End" indices, and for $\nicefrac{T}{2}+1<t_2\leq T$ it contains this basic unit. This means that in order to obtain the number of repetition of this basic unit in $\phi$, we must count the number of multiplications implemented by the chain of products in the above expression. Indeed this number is equal to:
		\begin{equation*}
		\sum_{t_L=\nicefrac{T}{2} +1}^T\sum_{t_{L-1}=\nicefrac{T}{2} +1}^{t_L}\cdots \sum_{t_2=\nicefrac{T}{2} +1}^{t_3}t_2 =\multiset{T/2}{L-1}
		\end{equation*}
		
	\end{proof}
	
	Finally, the form of the lower bound presented in Conjecture~\ref{conjecture:high_L} is obtained by considering a rank $R$ matrix, such as the one obtained by the Start-End matricization of the TN basic unit discussed above, raised to the Hadamard power of $\multiset{\nicefrac{T}{2}}{L-1}$. The rank of the resultant matrix, is upper bounded by {\tiny{$\multiset{R}{\multiset{\nicefrac{T}{2}}{L-1}}$}} as shown for example in~\citep{amini2012low}. We leave it as an open problem to prove Conjecture~\ref{conjecture:high_L}, by proving that the upper bound is indeed tight in this case.
	
	\newpage
	\section{Deferred proof of Theorem~\ref{theorem:main_result}}
	\label{app:proofs:main_result}
	
	In this section, we follow the proof strategy that is outlined in
	Section~\ref{sec:results}, and prove Theorem~\ref{theorem:main_result}, which
	shows a combinatorial advantage of deep recurrent networks over shallow ones in
	the ability to model long-term dependencies, as measured by the Start-End
	separation rank (see Section~\ref{sec:corr:sep}).
	In Appendices~\ref{app:proofs:main_result:shallow} and~\ref{app:proofs:main_result:deep}, we prove the bounds on the
	Start-End separation rank of the shallow and deep RACs, respectively,
	while more technical lemmas which are employed during the proof are relegated to
	Appendix~\ref{app:proofs:main_result:technical}.
	
	\subsection{The Start-End Separation Rank of Shallow RACs}
	\label{app:proofs:main_result:shallow}
	
	We consider the Tensor Network construction of the calculation carried
	out by a shallow RAC, given in Figure~\ref{fig:shallow_rac_tn}.
	According to the presented construction, the shallow RAC weights
	tensor [Equations~\eqref{eq:score} and~\eqref{eq:tt_decomp}] is represented by
	a Matrix Product State (MPS) Tensor Network
	\citep{orus2014practical}, with the following order-3 tensor
	building block:
	$M_{k_{t-1}d_tk_t}=W^\textrm{I}_{k_td_t}W^\textrm{H}_{k_tk_{t-1}}$,
	where $d_t\in[M]$ is the input index and $k_{t-1},k_t\in[R]$ are the
	internal indices (see
	Figure~\ref{fig:shallow_rac_tn}c). In
	TN terms, this means that the bond dimension of this MPS is equal to
	$R$. We apply the result of \citep{levine2018deep}, who state that the
	rank of the matrix obtained by matricizing any tensor according to a
	partition $(S,E)$ is equal to a min-cut separating $S$ from $E$ in the
	Tensor Network graph representing this tensor, for all of the values of the TN
	parameters but a set of Lebesgue measure zero. In this MPS Tensor
	Network, the minimal cut \wrt~the partition $(S,E)$ is equal to the
	bond dimension $R$, unless $R > M^{\nicefrac{T}{2}}$, in which
	case the minimal cut contains the external legs instead. Thus, in the
	TN representing $\A^{T,1, \Theta}_c$, the minimal cut \wrt~the partition
	$(S,E)$ is equal to $\min\{R,M^{\nicefrac{T}{2}}\}$, implying
	$\rank{\mat{\A^{T,1, \Theta}_c})_{S,E}} = \min\{R,M^{\nicefrac{T}{2}}\}$ for all values of the parameters but a set of Lebesgue measure zero.
	The first half of the
	theorem follows from applying Claim~\ref{claim:grid_sep_shallow},
	which assures us that the Start-End separation rank of the function
	realized by a shallow ($L=1$) RAC is equal to
	$\rank{\mat{\A^{T,1, \Theta}_c})_{S,E}}$.
	
	\hfill $\square$ 

	\subsection{Lower-bound on the Start-End Separation Rank of Deep RACs}
	\label{app:proofs:main_result:deep}
	
	For a deep network, Claim~\ref{claim:grid_sep_deep} assures us that
	the Start-End separation rank of the function realized by a depth
	$L =2$ RAC is lower bounded by the rank of the matrix obtained by the
	corresponding grid tensor matricization, for any choice of template
	vectors. Specifically:
	\begin{equation*}
	\sep{S,E}{y^{T,L, \Theta}_c}
	\geq \rank{\mat{\A(y^{T,L, \Theta}_c)}_{S,E}}.
	\end{equation*}
	Thus, proving that
	$\rank{\mat{\A(y^{T,L, \Theta}_c)}_{S,E}} \geq \multiset{\min\{R,M\}}{T/2}$
	for all of the values of parameters $\Theta\times\h^{0,l}$ but a set of Lebesgue measure
	zero, would satisfy the theorem.
	
	In the following, we provide an assignment of weight matrices and initial hidden states for
	which $\rank{\mat{\A(y^{T,L, \Theta}_c)}_{S,E}} = \multiset{\min\{R,M\}}{T/2}$.
	In accordance with Claim~\ref{claim:rank_everywhere}, this will
	suffice as such an assignment implies this rank is achieved for all configurations of
	the recurrent network weights but a set of Lebesgue measure zero.
	
	We begin by choosing a specific set of template vectors
	$\x^{(1)}, \ldots, \x^{(M)} \in \X$. Let $F \in \R^{M \times M}$ be a
	matrix with entries defined by $F_{ij} \equiv f_j(\x^{(i)})$.
	According to \citep{cohen2016convolutional}, since $\{f_d\}_{d=1}^M$
	are linearly independent, then there is a choice of template vectors
	for which $F$ is non-singular.
	
	Next, we describe our assignment. In the
	expressions below we use the notation $\delta_{ij} = \begin{cases}
	1 & i = j \\ 0 & i \neq j\end{cases}$. Let $z \in \R \setminus \{0\}$
	be an arbitrary non-zero real number, let $\Omega \in \R_+$ be an arbitrary
	positive real number, and let $Z \in\R^{R \times M}$
	be a matrix with entries
	$Z_{ij} \equiv \begin{cases} z^{\Omega^i \delta_{ij}} &
	i \leq M \\ 0 & i > M \end{cases}$.
	
	We set
	$\WI{1} \equiv Z \cdot (F^T)^{-1}$ and set $\WI{2}$ such that its entries
	are $\WI{2}_{ij} \equiv \delta_{i1}$. 
	We set
	$\WH{1} \equiv \WH{2} \equiv I$, i.e. to the identity matrix, and
	additionally we set the entries of $\WO$ to $\WO_{ij} = \delta_{1j}$. Finally, we choose the
	initial hidden state values so they bear no effect on the calculation, namely $\h^{0,l}=\left(\WH{l}\right)^{-1}\1=\1$ for $l=1,2$.
	
	Under the above assignment, the output for the corresponding class $c$
	after $T$ time-steps is equal to:
	\begin{align*}
	y^{T,L, \Theta}_c(\x^1, \ldots, \x^T)
	&= \left(\WO \h^{T,2}\right)_c \\
	(\WO_{ij} \equiv \delta_{1j}) \Rightarrow &= (\h^{T,2})_1 \\
	[\text{Equation~\eqref{eq:deep_rn}}] \Rightarrow
	&= \left((\WH{2} \h^{T-1,2}) \odot (\WI{2}\h^{T,1})\right)_1 \\
	(\WH{2} \equiv I) \Rightarrow
	&= \left((\h^{T-1,2}) \odot (\WI{2}\h^{T,1})\right)_1\\
	(\h^{0,2}=\1) \Rightarrow
	&= \prod_{t=1}^T \left(\WI{2}\h^{t,1}\right)_1 \\
	(\WI{2}_{ij} \equiv \delta_{1i}) \Rightarrow
	&= \prod_{t=1}^T \sum_{r=1}^R \left(\h^{t,1}\right)_r \\
	[\text{Equation~\eqref{eq:deep_rn}}] \Rightarrow
	&= \prod_{t=1}^T \sum_{r=1}^R
	\left((\WH{1} \h^{t-1,1}) \odot (\WI{1} \ff(\x^t))\right)_r \\
	(\WH{1} \equiv I) \Rightarrow
	&= \prod_{t=1}^T \sum_{r=1}^R
	\left((\h^{t-1,1}) \odot (\WI{1} \ff(\x^t))\right)_r\\
	(\h^{0,1}=\1) \Rightarrow
	&= \prod_{t=1}^T \sum_{r=1}^R \prod_{j=1}^t
	\left(\WI{1} \ff(\x^j)\right)_r.
	\end{align*}
	When evaluating the grid tensor for our chosen set of template
	vectors, i.e. $\A(y^{T,L, \Theta}_c)_{d_1,\ldots,d_T} =
	y^{T,L, \Theta}_c(\x^{(d_1)}, \ldots, \x^{(d_T)})$, we can
	substitute $f_j(\x^{(i)}) \equiv F_{ij}$, and thus
	\begin{equation*}
	(\WI{1} \ff(\x^{(d)}))_r = (\WI{1} F^T)_{rd} = (Z \cdot (F^T)^{-1}F^T)_{rd} = Z_{rd}.
	\end{equation*}
	Since we defined $Z$ such that for $r \geq \min\{R,M\}$ $Z_{rd} = 0$, and denoting $\bar{R}\equiv \min\{R,M\} $ for brevity of notation, the
	grid tensor takes the following form:
	\begin{equation*}
	\A(y^{T,L, \Theta}_c)_{d_1,\ldots,d_T} =
	\prod_{t=1}^T \sum_{r=1}^{\bar{R}} \prod_{j=1}^t Z_{r d_j}
	= \left(\prod_{t=1}^{\nicefrac{T}{2}} \sum_{r=1}^{\bar{R}} \prod_{j=1}^t Z_{r d_j}\right)
	\cdot \left(\prod_{t=\nicefrac{T}{2}+1}^{T} \sum_{r=1}^{\bar{R}} \prod_{j=1}^t Z_{r d_j}\right),
	\end{equation*}
	where we split the product into two expressions, the left
	part that contains only the indices in the start set $S$, i.e.
	$d_1,\ldots,d_{\nicefrac{T}{2}}$, and the right part which contains
	all external indices (in the start set $S$ and the end set $E$). Thus, under matricization w.r.t. the Start-End
	partition, the left part is mapped to a vector
	$\aaa \equiv \matflex{\prod_{t=1}^{\nicefrac{T}{2}} \sum_{r=1}^{\bar{R}}
		\prod_{j=1}^t Z_{r d_j}}_{S,E}$ containing only non-zero entries per the definition of $Z$, and
	the right part is mapped to a matrix $B \equiv
	\matflex{\prod_{t=\nicefrac{T}{2}+1}^T
		\sum_{r=1}^{\bar{R}} \prod_{j=1}^t Z_{r d_j}}_{S,E}$, where each entry
	of $\uu$ multiplies the corresponding row of $B$. This results in:
	\begin{equation*}
	\mat{\A(y^{T,L, \Theta}_c)_{d_1,\ldots,d_T}}_{S,E}
	= \mathrm{diag}(\aaa) \cdot B.
	\end{equation*}
	Since $\aaa$ contains only non-zero entries, $\mathrm{diag}(\aaa)$ is
	of full rank, and so
	$\rank{\mat{\A(y^{T,L, \Theta}_c)_{d_1,\ldots,d_T}}_{S,E}}=\rank{B}$, leaving us to prove that $\rank{B} =
	\multiset{{\bar{R}}}{\nicefrac{T}{2}}$ . For brevity of notation, we define $N \equiv \multiset{{\bar{R}}}{\nicefrac{T}{2}}$.
	
	To prove the above, it is sufficient to show that $B$
	can be written as a sum of $N$
	rank-1 matrices, i.e. $B=\sum_{i=1}^N \uu^{(i)} \otimes \vv^{(i)}$,
	and that $\{\uu^{(i)}\}_{i=1}^N$ and $\{\vv^{(i)}\}_{i=1}^N$ are two
	sets of linearly independent vectors. Indeed, applying
	Claim~\ref{claim:decomp} on the entries of $B$, specified w.r.t. the
	row $(d_1,\ldots,d_{\nicefrac{T}{2}})$ and column
	$(d_{\nicefrac{T}{2}+1},\ldots,d_T)$, yields the following form:
	\begin{equation*}
	B_{(S,E)}=
	\sum_{\substack{\p^{(\nicefrac{T}{2})} \in \state{{\bar{R}}}{\nicefrac{T}{2}}}}
	\left(
	\prod_{r=1}^{{\bar{R}}} \prod_{j=1}^{\nicefrac{T}{2}} Z_{r d_j}^{p^{(\nicefrac{T}{2})}_r}
	\vphantom{\sum_{\substack{(\p^{(\nicefrac{T}{2}-1)},\ldots,\p^{(1)}) \\ \in \trajectory{\p^{(\nicefrac{T}{2})}}}}}
	\right) \cdot
	\left(
	\sum_{\substack{(\p^{(\nicefrac{T}{2}-1)},\ldots,\p^{(1)}) \\ \in \trajectory{\p^{(\nicefrac{T}{2})}}}}
	\prod_{r=1}^{{\bar{R}}}{\prod_{j=\nicefrac{T}{2}+1}^T} Z_{r d_j}^{p_r^{(T-j+1)}}
	\right),
	\end{equation*}
	where for all $k$, $\p^{(k)}$ is $\bar{R}$-dimensional vector of non-negative integer
	numbers which sum to $k$, and we explicitly define
	$\state{\bar{R}}{\nicefrac{T}{2}}$ and $\trajectory{\p^{(\nicefrac{T}{2})}}$
	in Claim~\ref{claim:decomp}, providing a softer more intuitive definition hereinafter. $\state{\bar{R}}{\nicefrac{T}{2}}$
	can be viewed as the set of all possible states of a bucket containing $\nicefrac{T}{2}$
	balls of ${\bar{R}}$ colors, where $p^{(\nicefrac{T}{2})}_r$ for $r\in[\bar{R}]$ specifies the number of balls
	of the $r$'th color.
	$\trajectory{\p^{(\nicefrac{T}{2})}}$ can be viewed
	as all possible trajectories from a given state to an empty bucket, i.e. $(0,\ldots,0)$,
	where at each step we remove a single ball from the bucket.
	We note that the number of all initial states of the bucket is exactly
	$\abs{\state{\bar{R}}{\nicefrac{T}{2}}} = N \equiv \multiset{{\bar{R}}}{\nicefrac{T}{2}}$.
	Moreover, since the expression in the left parentheses
	contains solely indices from the start set $S$, i.e. $d_1,\ldots,d_{\nicefrac{T}{2}}$,
	while the right contains solely indices from the end set $E$, i.e.
	$d_{\nicefrac{T}{2}+1},\ldots,d_T$, then each summand is in fact a rank-1 matrix. Specifically,
	it can be written as $\uu^{\p^{(\nicefrac{T}{2})}} \otimes \vv^{\p^{(\nicefrac{T}{2})}}$,
	where the entries of $\uu^{\p^{(\nicefrac{T}{2})}}$ are represented by the expression
	in the left parentheses, and those of $\vv^{\p^{(\nicefrac{T}{2})}}$ by the expression
	in the right parentheses.
	
	We prove that the set $\left\{ \uu^{\p^{(\nicefrac{T}{2})}} \in \R^{M^{\nicefrac{T}{2}}}
	\right\}_{\p^{(\nicefrac{T}{2})} \in \state{\bar{R}}{\nicefrac{T}{2}}}$ is linearly
	independent by arranging it as the columns of the matrix
	$U \in \R^{M^{\nicefrac{T}{2}} \times N}$, and showing that its rank equals to $N$.
	Specifically, we observe the sub-matrix defined by the subset of the rows of $U$,
	such that we select the row $\dd \equiv (d_1,\ldots,d_{\nicefrac{T}{2}})$ only if it holds
	that $\forall j, d_j \leq d_{j+1}$. 
	Note that there are exactly $N$ such rows, similarly to the number of columns, which can be intuitively understood since for the imaginary `bucket states' defining the columns $\p^{(\nicefrac{T}{2})}$ there is no meaning of order in the balls, and having imposed the restriction $\forall j, d_j \leq d_{j+1}$ on the $\nicefrac{T}{2}$ length tuple $\dd$, there is no longer a degree of freedom to order the `colors' in $\dd$, reducing the number of rows from $M^{\nicefrac{T}{2}}$ to $N$ (note that by definition $N\leq\multiset{M}{\nicefrac{T}{2}}<M^{\nicefrac{T}{2}}$ ). 
	Thus, in the resulting
	sub-matrix, denoted by $\bar{U} \in \R^{N \times N}$, not only do the columns correspond
	to the vectors of $\state{\bar{R}}{\nicefrac{T}{2}}$, but also its rows, where
	the row specified by the tuple $\dd$, corresponds to the
	vector $\q^{(\nicefrac{T}{2})} \in \state{\bar{R}}{\nicefrac{T}{2}}$, such that for $r\in[\bar{R}]:~q^{(\nicefrac{T}{2})}_r \equiv \abs{\{j \in [\nicefrac{T}{2}] | d_j = r\}}$ specifies the amount of
	repetitions of the number (`color') $r$ in the given tuple.
	
	Accordingly, for each element of
	$\bar{U}$ the following holds:
	\begin{align*}
	\bar{U}_{\q^{(\nicefrac{T}{2})},\p^{(\nicefrac{T}{2})}}
	&= \prod_{r=1}^{{\bar{R}}} \prod_{j=1}^{\nicefrac{T}{2}} Z_{r d_j}^{p^{(\nicefrac{T}{2})}_r}\\
	(Z_{ij} = z^{\Omega^i \delta_{ij}}) \Rightarrow
	&= z^{\sum_{j=1}^{\nicefrac{T}{2}} \sum_{r=1}^{\bar{R}} p^{(\nicefrac{T}{2})}_r\Omega^r\delta_{r d_j}} \\
	(\text{definition of }\delta_{ij}) \Rightarrow
	&= z^{\sum_{j=1}^{\nicefrac{T}{2}} \Omega^{d_j} p^{(\nicefrac{T}{2})}_{d_j}} \\
	(\text{Grouping identical summands}) \Rightarrow
	&= z^{\sum_{r=1}^{\bar{R}} \Omega^r \abs{\{j \in [\nicefrac{T}{2}] | d_j = r\}} p^{(\nicefrac{T}{2})}_r} \\
	(q^{(\nicefrac{T}{2})}_r \equiv \abs{\{j \in [\nicefrac{T}{2}] | d_j = r\}}) \Rightarrow
	&= z^{\sum_{r=1}^{\bar{R}} \Omega^r q^{(\nicefrac{T}{2})}_r p^{(\nicefrac{T}{2})}_r} \\
	\left(
	\begin{matrix}
	\bar{q}^{(\nicefrac{T}{2})}_r {\equiv} \Omega^{\nicefrac{r}{2}} q^{(\nicefrac{T}{2})}_r \\
	\bar{p}^{(\nicefrac{T}{2})}_r {\equiv} \Omega^{\nicefrac{r}{2}} p^{(\nicefrac{T}{2})}_r
	\end{matrix}
	\right) \Rightarrow
	&= z^{\inprod{\bar{\q}^{(\nicefrac{T}{2})}}{\bar{\p}^{(\nicefrac{T}{2})}}}.
	\end{align*}
	Since the elements of $\bar{U}$ are polynomial in $z$, then as we prove in
	Lemma~\ref{lemma:poly_full_rank}, it is sufficient to show that there exists
	a single contributor to the determinant of $\bar{U}$ that has the highest degree
	of $z$ in order to ensure that the matrix is fully ranked for all values of $z$
	but a finite set. Observing the summands of the determinant, i.e.
	$ z^{\sum_{\q^{(\nicefrac{T}{2})} \in \state{\bar{R}}{\nicefrac{T}{2}}}
		\inprod{\bar{\q}{(\nicefrac{T}{2})}}{\sigma(\bar{\q}^{(\nicefrac{T}{2})})}}$,
	where $\sigma$ is a permutation on the rows of $\bar{U}$, and noting that $\state{\bar{R}}{\nicefrac{T}{2}}$
	is a set of non-negative numbers by definition,
	Lemma~\ref{lemma:rearrange} assures us the existence of a strictly maximal
	contributor, satisfying the conditions of Lemma~\ref{lemma:poly_full_rank}, thus the set $\left\{ \uu^{\p^{(\nicefrac{T}{2})}} 
	\right\}_{\p^{(\nicefrac{T}{2})} \in \state{\bar{R}}{\nicefrac{T}{2}}}$ is linearly
	independent.
	
	We prove that the set $\left\{ \vv^{\p^{(\nicefrac{T}{2})}} \in \R^{M^{\nicefrac{T}{2}}}
	\right\}_{\p^{(\nicefrac{T}{2})} \in \state{\bar{R}}{\nicefrac{T}{2}}}$ is linearly
	independent by arranging it as the columns of the matrix
	$V \in \R^{M^{\nicefrac{T}{2}} \times N}$, and showing that its rank equals to $N$.
	As in the case of $U$, we select the same sub-set of rows to form the sub-matrix
	$\bar{V} \in \R^{N \times N}$. We show that each of the diagonal elements of $\bar{V}$
	is a polynomial function whose degree is strictly larger than the degree of all other
	elements in its row. As an immediate consequence, the product of the diagonal elements,
	i.e. $\prod_{i=1}^N \bar{V}_{ii}(z)$, has degree strictly larger than any other summand
	of the determinant $\det(\bar{V})$, and by employing Lemma~\ref{lemma:poly_full_rank},
	$\bar{V}$ has full-rank for all values of $z$
	but a finite set. The degree of the polynomial function in each entry of $\bar{V}$
	is given by:
	\begin{align*}
	\deg\left(\bar{V}_{\dd,\p^{(\nicefrac{T}{2})}} \right)
	&= \max_{\substack{(\p^{(\nicefrac{T}{2}-1)},\ldots,\p^{(1)}) \\ \in \trajectory{\p^{(\nicefrac{T}{2})}}}}
	\deg\left( \prod_{r=1}^{{\bar{R}}}{\prod_{j=\nicefrac{T}{2}+1}^T} Z_{r d_j}^{p_r^{(T-j+1)}}\right) \\
	&= \max_{\substack{(\p^{(\nicefrac{T}{2}-1)},\ldots,\p^{(1)}) \\ \in \trajectory{\p^{(\nicefrac{T}{2})}}}}
	\deg\left( z^{
		\sum_{j=\nicefrac{T}{2}+1}^T \sum_{r=1}^{\bar{R}}
		\Omega^r p^{(T-j+1)}_r \delta_{r d_j}
	}\right) \\
	&= \max_{\substack{(\p^{(\nicefrac{T}{2}-1)},\ldots,\p^{(1)}) \\ \in \trajectory{\p^{(\nicefrac{T}{2})}}}}
	\sum_{j=\nicefrac{T}{2}+1}^T
	\Omega^{d_j} p^{(T-j+1)}_{d_j}.
	\end{align*}
	The above can be formulated as the following combinatorial optimization problem. We are
	given an initial state $\p^{(\nicefrac{T}{2})}$ of the bucket of $\nicefrac{T}{2}$ balls
	of $\bar{R}$ colors and a sequence of colors $\dd = (d_{\nicefrac{T}{2}+1},\ldots,d_T)$.
	At time-step $j$ one ball is taken out of the bucket and yields a reward of
	$\Omega^{d_j} p^{(T-j+1)}_{d_j}$, i.e. the number of remaining balls
	of color $d_j$ times the weight $\Omega^{d_j}$. Finally,
	$\deg(\bar{V}_{\dd, \p^{(\nicefrac{T}{2})}})$ is the accumulated reward of the optimal
	strategy of emptying the bucket. In Lemma~\ref{lemma:combinatoric_induction} we prove
	that there exists a value of $\Omega$ such that for every sequence of colors $\dd$,
	i.e. a row of $\bar{V}$, the maximal reward over all possible initial states is solely attained at the state $\q^{(\nicefrac{T}{2})}$
	corresponding to $\dd$, i.e.
	$q^{(\nicefrac{T}{2})}_r = \abs{\{j\in\{\nicefrac{T}{2}+1,\ldots,T\}|d_j = r\}}$.
	Hence, $\deg(\bar{V}_{ii})$ is indeed strictly larger than the degree of all
	other elements in the $i$'th row $\forall i \in [N]$.
	
	Having proved that both $U$ and $V$ have rank $N \equiv \multiset{\bar{R}}{\nicefrac{T}{2}}$ for all values of $z$
	but a finite set, we know there exists a value of $z$ for which $\rank{B} =\multiset{\bar{R}}{\nicefrac{T}{2}}$, and the theorem follows.
	
	\hfill $\square$ 
	
	\subsection{Technical Lemmas and Claims}
	\label{app:proofs:main_result:technical}
	
	In this section we prove a series of useful technical lemmas, that we have
	employed in our proof for the case of deep RACs, as described in
	Appendix~\ref{app:proofs:main_result:deep}.
	We begin by quoting a claim regarding the prevalence of the maximal matrix rank
	for matrices whose entries are polynomial functions:
	\begin{claim} \label{claim:rank_everywhere}
		Let $M, N, K \in \N$, $1 \leq r \leq \min\{M,N\}$ and a polynomial mapping
		$A:\R^K \to \R^{M \times N}$, i.e. for every $i \in [M]$ and $j\in [N]$ it
		holds that $A_{ij}:\R^K \to \R$ is a polynomial function. If there exists a
		point $\x \in \R ^K$ s.t. ${\textrm {rank}}{(A(\x))} \geq r$, then the set
		$\{\x \in \R^K : \textrm{rank}{(A(\x))} < r\}$ has zero measure (w.r.t. the
		Lebesgue measure over $\R^K$).
	\end{claim}
	\begin{proof}
		See \citep{sharirtractable}.
	\end{proof}
	
	Claim~\ref{claim:rank_everywhere} implies that it suffices to show a specific
	assignment of the recurrent network weights for which the corresponding grid
	tensor matricization achieves a certain rank, in order to show this is a lower
	bound on its rank for all configurations of the network weights but a set of
	Lebesgue measure zero. Essentially, this means that it is enough to provide a
	specific assignment that achieves the required bound in
	Theorem~\ref{theorem:main_result} in order to prove the theorem. Next, we show
	that for a matrix with entries that are polynomials in $x$, if a single
	contributor to the determinant has the highest degree of $x$, then the matrix is
	fully ranked for all values of $x$ but a finite set:
	
	\begin{lemma}\label{lemma:poly_full_rank}
		Let $A\in\R^{N\times N}$ be a matrix whose entries are polynomials in
		$x\in\R$. In this case, its determinant may be written as
		$\det(A)=\sum_{\sigma\in S_N}sgn(\sigma)p_\sigma(x)$, where $S_N$ is the
		symmetric group on $N$ elements and $p_\sigma(x)$ are polynomials defined by
		$p_\sigma(x)\equiv\prod_{i=1}^{N} A_{i\sigma(i)}(x),~\forall{\sigma\in S_n}$.
		Additionally, let there exist $\bar{\sigma}$ such that
		$\deg(p_{\bar{\sigma}}(x)) > \deg(p_{\sigma}(x)) ~ \forall \sigma
		\neq \bar{\sigma}$. Then, for all values of $x$ but a finite set, $A$ is
		fully ranked.
	\end{lemma}
	\begin{proof}
		We show that in this case $\det(A)$, which is a polynomial in $x$ by its
		definition, is not the zero polynomial. Accordingly, $\det(A)\neq 0$ for all
		values of $x$ but a finite set. Denoting $t\equiv\deg(p_{\bar{\sigma}}(x))$,
		since $t>\deg(p_{\sigma}(x))~\forall\sigma\neq\bar{\sigma}$, a monomial of
		the form $c\cdot x^t,c\in\R \setminus \{0\}$ exists in $p_{\bar{\sigma}}(x)$ and
		doesn't exist in any $p_\sigma(x),~\sigma\neq\bar{\sigma}$. This implies
		that $\det(A)$ is not the zero polynomial, since its leading term has a
		non-vanishing coefficient $sgn(\bar{\sigma})\cdot c\neq 0$, and the lemma
		follows from the basic identity: $\det(A)\neq 0 \iff$ $A$ is fully ranked.
	\end{proof}
	
	The above lemma assisted us in confirming that the assignment provided for
	the recurrent network weights indeed achieves the required grid tensor
	matricization rank of $\multiset{\bar{R}}{T/2}$. The following lemma, establishes a
	useful relation we refer to as the \emph{vector rearrangement inequality}:
	\begin{lemma}\label{lemma:rearrange}
		Let $\{\vv^{(i)}\}_{i=1}^{N}$ be a set of $N$ different vectors in $\R^{\bar{R}}$
		such that $\forall i\in[N],~j\in[\bar{R}]:~v^{(i)}_j\geq 0$. Then, for all
		$\sigma\in S_N$ such that $\sigma\neq\mathbb{I}_N$, where $S_N$ is the
		symmetric group on $N$, it holds that:
		\begin{equation*}
		\sum_{i=1}^N \inprod{\vv^{(i)}}{\vv^{(\sigma(i))}} < \sum_{i=1}^{N} \norm{\vv^{(i)}}^2.
		\end{equation*}
	\end{lemma}
	\begin{proof}
		We rely on theorem 368 in~\citep{hardy1952inequalities}, which implies that
		for a set of non-negative numbers $\{a^{(1)},\ldots,a^{(N)}\}$ the following
		holds for all $\sigma\in S_N$:
		\begin{equation}\label{eq:rearrange}
		\sum_{i=1}^{N}a^{(i)}a^{({\sigma(i)})}\leq\sum_{i=1}^{N}(a^{(i)})^2,
		\end{equation}
		with equality obtained only for $\sigma$ which upholds
		$\sigma(i)=j\iff a^{(i)}=a^{(j)}$. The above relation, referred to as the
		\emph{rearrangement inequality}, holds separately for each component
		$j\in[\bar{R}]$ of the given vectors:
		\begin{equation*}
		\sum_{i=1}^{N}v_j^{(i)}v_j^{(\sigma(i))}\leq\sum_{i=1}^{N}(v_j^{(i)})^2.
		\end{equation*}
		We now prove that for all $\sigma\in S_N$ such that
		$\sigma\neq\mathbb{I}_N$, $\exists \hat{j}\in[\bar{R}]$ for which the above inequality
		is hard, \ie:
		\begin{equation}\label{hard_ineq}
		\sum_{i=1}^{N}v_{\hat{j}}^{(i)}v_{\hat{j}}^{(\sigma(i))}<\sum_{i=1}^{N}(v_{\hat{j}}^{(i)})^2.
		\end{equation}
		By contradiction, assume that $\exists\hat{\sigma}\neq\mathbb{I}_N$ for
		which $\forall j \in [\bar{R}]$:
		\begin{equation*}
		\sum_{i=1}^{N}v_j^{(i)}v_j^{(\hat{\sigma}(i))}=\sum_{i=1}^{N}(v_j^{(i)})^2.
		\end{equation*}
		From the conditions of achieving equality in the rearrangement inequality
		defined in Equation~\eqref{eq:rearrange}, it holds that
		$\forall j \in [\bar{R}]:~v_j^{(\hat{\sigma}(i))}= v_j^{(i)}$, trivially
		entailing: $\vv^{(\hat{\sigma}(i))}=\vv^{(i)}$. Thus,
		$\hat{\sigma}\neq\mathbb{I}_N$ would yield a contradiction to
		$\{\vv^{(i)}\}_{i=1}^{N}$ being a set of $N$ different vectors in $\R^{\bar{R}}$.
		Finally, the hard inequality of the lemma for $\sigma\neq\mathbb{I}_N$ is
		implied from Equation~\eqref{hard_ineq}:
		\begin{equation*}
		\sum_{i=1}^N \inprod{\vv^{(i)}}{\vv^{(\sigma(i))}}
		\equiv \sum_{i=1}^N \left(\sum_{j=1}^{\bar{R}} v_j^{(i)} v_j^{(\sigma(i))}\right)
		= \sum_{j=1}^{\bar{R}} \left(\sum_{i=1}^N v_j^{(i)} v_j^{(\sigma(i))}\right)
		< \sum_{j=1}^{\bar{R}} \left(\sum_{i=1}^N (v_j^{(i)})^2 \right)
		= \sum_{i=1}^N \norm{\vv^{(i)}}^2.
		\end{equation*}
	\end{proof}

	The vector rearrangement inequality in Lemma~\ref{lemma:rearrange}, helped us
	ensure that our matrix of interest denoted $\bar{U}$ upholds the conditions of
	Lemma~\ref{lemma:poly_full_rank} and is thus fully ranked. Below, we show an identity that allowed us to
	make combinatoric sense of a convoluted expression:
	\begin{claim}\label{claim:decomp}
		Let $\bar{R}$ and $M$ be positive integers, let $Z \in \R^{\bar{R}\times M}$
		be a matrix, and let $\A$ be a tensor with $T$
		modes, each of dimension $M$, defined by $\A_{d_1,\ldots,d_T} \equiv
		\prod_{t=\nicefrac{T}{2}+1}^{T}\sum_{r=1}^{\bar{R}}\prod_{j=1}^{t}Z_{rd_j}$,
		where $d_1,\ldots,d_T \in [M]$. Then, the following identity holds:
		\begin{align*}
		\A_{d_1,\ldots,d_T} &=
		\sum_{\substack{\p^{(\nicefrac{T}{2})} \\ \in \state{\bar{R}}{\nicefrac{T}{2}}}}
		\sum_{\substack{(\p^{(\nicefrac{T}{2}-1)},\ldots,\p^{(1)}) \\ \in \trajectory{\p^{(\nicefrac{T}{2})}}}}
		\prod_{r=1}^{\bar{R}}
		\left(\prod_{j=1}^{\nicefrac{T}{2}} Z_{r d_j}^{p^{(\nicefrac{T}{2})}_r}\right)
		\left(\prod_{j=\nicefrac{T}{2}+1}^T Z_{r d_j}^{p_r^{(T-j+1)}}\right),
		\end{align*}
		where $\state{\bar{R}}{K} \equiv \{\p^{(K)} \in (\N \cup \{0\})^{\bar{R}} |
		\sum_{i=1}^{\bar{R}} p_i = K\}$,
		and $\trajectory{\p^{(K)}} \equiv \{ (\p^{(K-1)}, \ldots, \p^{(1)}) |
		\forall k \in [K-1], (\p^{(k)} \in \state{\bar{R}}{k} \wedge
		\forall r \in [\bar{R}], p^{(k)}_r \leq p_r^{(k+1)})\}$.
		\footnote{See Appendix~\ref{app:proofs:main_result:deep} for a more intuitive
			definition of the sets $\state{\bar{R}}{K}$ and
			$\trajectory{\p^{(T-k+1)}}$.}
	\end{claim}
	\begin{proof}
		We will prove the following more general identity by induction. For any
		$k \in [T]$, define $\A^{(k)}_{d_1,\ldots,d_T} \equiv
		\prod_{t=k}^{T}\sum_{r=1}^{\bar{R}}\prod_{j=1}^{t}Z_{rd_j}$, then the
		following identity holds:
		\begin{align*}
		\A^{(k)}_{d_1,\ldots,d_T} &=
		\smashoperator[l]{\sum_{\substack{\p^{(T-k+1)} \\ \in \state{\bar{R}}{T-k+1}}}}
		\sum_{\substack{(\p^{(T-k)},\ldots,\p^{(1)}) \\ \in \trajectory{\p^{(T-k+1)}}}}
		\prod_{r=1}^{\bar{R}}
		\left(\prod_{j=1}^{k-1} Z_{r d_j}^{p^{(T-k+1)}_r}\right)
		\left(\smashoperator[r]{\prod_{j=k}^T} Z_{r d_j}^{p_r^{(T-j+1)}}\right).
		\end{align*}
		
		The above identity coincides with our claim for $k=\nicefrac{T}{2}+1$
		We begin with the base case of $k=T$, for which the set $\state{\bar{R}}{1}$
		simply equals to the unit vectors of $(\N\cup\{0\})^{\bar{R}}$, i.e. for
		each such $\p^{(1)}$ there exists $\bar{r} \in [\bar{R}]$ such that
		$p^{(1)}_r = \delta_{\bar{r} r} \equiv
		\begin{cases} 1 & \bar{r} = r \\ 0 & \bar{r} \neq r \end{cases}$. Thus,
		the following equalities hold:
		\begin{align*}
		\sum_{\p^{(1)} \in \state{\bar{R}}{1}}
		\prod_{r=1}^{\bar{R}} \prod_{j=1}^T Z_{r d_j}^{p^{(1)}_r}
		&= \sum_{\bar{r}=1}^{\bar{R}} \prod_{r=1}^{\bar{R}} \prod_{j=1}^T
		Z_{r d_j}^{\delta_{\bar{r} r}}
		= \sum_{\bar{r}=1}^{\bar{R}} \prod_{j=1}^T Z_{\bar{r} d_j}
		= \A^{(T)}_{d_1,\ldots,d_T}.
		\end{align*}
		
		By induction on $k$, we assume that the claim holds for $\A^{(k+1)}$ and
		prove it on $\A^{(k)}$. First notice that we can rewrite our claim for
		$k < T$ as:
		\begin{align}\label{eq:proof:decomp:alternative}
		\A^{(k)}_{d_1,\ldots,d_T} &=
		\smashoperator[l]{\sum_{\substack{\p^{(T-k+1)} \\ \in \state{\bar{R}}{T-k+1}}}}
		\sum_{\substack{(\p^{(T-k)},\ldots,\p^{(1)}) \\ \in \trajectory{\p^{(T-k+1)}}}}
		\prod_{r=1}^{\bar{R}}
		\left(\prod_{j=1}^{k} Z_{r d_j}^{p^{(T-k+1)}_r}\right)
		\left(\smashoperator[r]{\prod_{j=k+1}^T} Z_{r d_j}^{p_r^{(T-j+1)}}\right),
		\end{align}
		where we simply moved the $k$'th term $Z^{p^{(k)}_r}_{r d_k}$ in the right
		product expression to the left product. We can also can rewrite $\A^{(k)}$
		as a recursive formula:
		\begin{align*}
		\A^{(k)}_{d_1,\ldots,d_T} &=
		\left(\sum_{r=1}^{\bar{R}} \prod_{j=1}^k Z_{r d_j}\right)
		\cdot \A^{(k+1)}_{d_1,\ldots,d_T}
		= \left(\sum_{\bar{r}=1}^{\bar{R}} \prod_{r=1}^{\bar{R}} \prod_{j=1}^k
		Z_{r d_j}^{\delta_{\bar{r} r}}\right)
		\cdot \A^{(k+1)}_{d_1,\ldots,d_T}
		\end{align*}.
		Then, employing our induction assumption for $\A^{(k+1)}$, results in:
		\begin{align}
		\A^{(k)}_{d_1,\ldots,d_T}
		&=
		\left(\sum_{\bar{r}=1}^{\bar{R}} \prod_{r=1}^{\bar{R}} \prod_{j=1}^k
		Z_{r d_j}^{\delta_{\bar{r} r}}\right)\quad
		\smashoperator[l]{\sum_{\substack{\p^{(T-k)} \\ \in \state{\bar{R}}{T-k}}}}
		\sum_{\substack{(\p^{(T-k-1)},\ldots,\p^{(1)}) \\ \in \trajectory{\p^{(T-k)}}}}
		\prod_{r=1}^{\bar{R}}
		\left(\prod_{j=1}^k Z_{r d_j}^{p^{(T-k)}_r}\right)
		\left(\smashoperator[r]{\prod_{j=k+1}^T} Z_{r d_j}^{p_r^{(T-j+1)}}\right)
		\nonumber \\
		&=
		\sum_{\bar{r}=1}^{\bar{R}}
		{\sum_{\substack{\p^{(T-k)} \\ \in \state{\bar{R}}{T-k}}}}
		\sum_{\substack{(\p^{(T-k-1)},\ldots,\p^{(1)}) \\ \in \trajectory{\p^{(T-k)}}}}
		\prod_{r=1}^{\bar{R}}
		\left(\prod_{j=1}^k Z_{r d_j}^{p^{(T-k)}_r + \delta_{\bar{r} r}}\right)
		\left(\smashoperator[r]{\prod_{j=k+1}^T} Z_{r d_j}^{p_r^{(T-j+1)}}\right)
		\label{eq:proof:decomp:almost}
		\end{align}
		To prove that the right hand side of Equation~\eqref{eq:proof:decomp:almost} is
		equal to our alternative form of our claim given by
		Equation~\eqref{eq:proof:decomp:alternative}, it is sufficient to show a
		bijective mapping from the terms in the sum of
		Equation~\eqref{eq:proof:decomp:almost}, each specified by a sequence
		$(\bar{r}, \p^{(T-k)}, \ldots, \p^{(1)})$, where $\bar{r} \in [\bar{R}]$,
		$\p^{(T-k)} \in \state{\bar{R}}{T-k}$, and $(\p^{(T-k-1)}, \ldots, \p^{(1)})
		\in \trajectory{\p^{(T-k)}}$, to the terms in the sum of
		Equation~\eqref{eq:proof:decomp:alternative}, each specified by a similar sequence
		$(\p^{(T-k+1)}, \p^{(T-k)}, \ldots, \p^{(1)})$, where $\p^{(T-k+1)} \in
		\state{\bar{R}}{T-k+1}$ and $(\p^{(T-k)},\ldots,\p^{(1)}) \in
		\trajectory{\p^{(T-k+1)}}$.
		
		Let $\phi$ be a mapping such that
		$(\bar{r}, \p^{(T-k)}, \ldots, \p^{(1)}) \overset{\phi}{\mapsto}
		(\p^{(T-k+1)}, \p^{(T-k)}, \ldots, \p^{(1)})$, where
		$p^{(T-k+1)}_r \equiv p^{(T-k)}_r + \delta_{\bar{r} r}$.
		$\phi$ is injective, because if $\phi(\bar{r}_1, \p^{(T-k, 1)}, \ldots,
		\p^{(1, 1)}) = \phi(\bar{r}_2, \p^{(T-k, 2)}, \ldots, \p^{(1, 2)})$ then
		for all $j \in \{1,\ldots,T-k+1\}$ it holds that $\p^{(j,1)} = \p^{(j,2)}$,
		and specifically for $\p^{(T-k+1, 1)} = \p^{(T-k+1, 2)}$ it entails that
		$\delta_{\bar{r}_1 r} = \delta_{\bar{r}_2 r}$, and thus
		$\bar{r}_1 = \bar{r}_2$.
		$\phi$ is surjective, because for any sequence
		$(\p^{(T-k+1)}, \p^{(T-k)}, \ldots, \p^{(1)})$, for which it holds that
		$\forall j, \p^{(j)} \in (\N \cup \{0\})^{\bar{R}}$,
		$\sum_{r=1}^{\bar{R}} p^{(j)}_r = j$, and
		$\forall r, p^{(j)}_r \leq p^{(j+1)}_r$, then it must also holds that
		$p^{(T-k+1)}_r - p^{(T-k)}_r = \delta_{\bar{r} r}$ for some $\bar{r}$,
		since $\sum_{r=1}^{\bar{R}} (p^{(T-k+1)}_r -p^{(T-k)}_r) = (T-k+1)-(T-k) =1$
		and every summand is a non-negative integer.
		
	\end{proof}
	
	Finally, Lemma~\ref{lemma:combinatoric_induction} assists us in ensuring that our matrix of interest denoted $\bar{V}$ upholds the conditions of
	Lemma~\ref{lemma:poly_full_rank} and is thus fully ranked:
	\begin{lemma}\label{lemma:combinatoric_induction}
		Let $\Omega \in \R_+$ be a positive real number.
		For every $\p^{(\nicefrac{T}{2})} \in \state{\bar{R}}{\nicefrac{T}{2}}$
		(see definition in Claim~\ref{claim:decomp}) and every
		$\dd = (d_{\nicefrac{T}{2}+1},\ldots,d_T) \in [\bar{R}]^{\nicefrac{T}{2}}$,
		where $\forall j, d_j \leq d_{j+1}$, we define the following optimization
		problem:
		\begin{equation*}
		f(\dd,\p^{(\nicefrac{T}{2})}) = \max_{\substack{(\p^{(\nicefrac{T}{2}-1)},\ldots,\p^{(1)}) \\ \in \trajectory{\p^{(\nicefrac{T}{2})}}}}
		\sum_{j=\nicefrac{T}{2}+1}^T
		\Omega^{d_j} p^{(T-j+1)}_{d_j},
		\end{equation*}
		where $\trajectory{\p^{(\nicefrac{T}{2})}}$ is defined as in Claim~\ref{claim:decomp}.
		Then, there exists $\Omega$ such that for every such
		$\dd$ the maximal value of $f(\dd,\p^{(\nicefrac{T}{2})})$ over all $\p^{(\nicefrac{T}{2})} \in \state{\bar{R}}{\nicefrac{T}{2}}$ is
		strictly attained at $\hat{\p}^{(\nicefrac{T}{2})}$ defined by
		$\hat{p}^{(\nicefrac{T}{2})}_r = \abs{\{j \in \{\nicefrac{T}{2}+1,\ldots,T\} | d_j = r\}}$.
	\end{lemma}
	\begin{proof}
		We will prove the lemma by first considering a simple strategy for choosing
		the trajectory for the case of $f(\dd, \hat{\p}^{(\nicefrac{T}{2})})$, achieving
		a certain reward $\rho^*$, and then showing that it is strictly larger than the
		rewards attained for all of the possible trajectories of any other
		$\p^{(\nicefrac{T}{2})} \neq \hat{\p}^{(\nicefrac{T}{2})}$.
		
		Our basic strategy is to always pick the ball of the lowest available color $r$.
		More specifically, if $\hat{p}^{(\nicefrac{T}{2})}_1 > 0$, then in the first
		$\hat{p}^{(\nicefrac{T}{2})}_1$ time-steps we remove balls of the color $1$,
		in the process of which we accept a reward of
		$\Omega^1 \hat{p}^{(\nicefrac{T}{2})}_1$ in the first time-step,
		$\Omega^1 (\hat{p}^{(\nicefrac{T}{2})}_1 - 1)$ in the second time-step,
		and so on to a total reward of
		$\Omega^1 \sum_{i=1}^{\hat{p}^{(\nicefrac{T}{2})}_1} i$.
		Then, we proceed to removing $\hat{p}^{(\nicefrac{T}{2})}_2$ balls of color $2$, and
		so forth. This strategy will result in an accumulated reward of:
		\begin{equation*}
		\rho^* \equiv \sum_{r=1}^{\bar{R}} \Omega^r
		\sum_{i=1}^{\hat{p}^{(\nicefrac{T}{2})}_r} i.
		\end{equation*}
		
		Next, we assume by contradiction that there exists
		$\p^{(\nicefrac{T}{2})} \neq \hat{\p}^{(\nicefrac{T}{2})}$ such that
		$\rho \equiv f(\dd, \p^{(\nicefrac{T}{2})}) \geq \rho^*$. We show by induction that this
		implies $\forall r, p^{(\nicefrac{T}{2})}_r \geq \hat{p}^{(\nicefrac{T}{2})}_r$,
		which would result in a contradiction, since per our assumption $\p^{(\nicefrac{T}{2})} \neq \hat{\p}^{(\nicefrac{T}{2})}$ this means that there is $r$ such
		that $p^{(\nicefrac{T}{2})}_r > \hat{p}^{(\nicefrac{T}{2})}_r$, but since
		$\p^{(\nicefrac{T}{2})}, \hat{\p}^{(\nicefrac{T}{2})} \in \state{\bar{R}}{\nicefrac{T}{2}}$ then the following contradiction arises
		$\nicefrac{T}{2} = \sum_{r=1}^{\bar{R}} p^{(\nicefrac{T}{2})}_r >
		\sum_{r=1}^{\bar{R}} \hat{p}^{(\nicefrac{T}{2})}_r = \nicefrac{T}{2}$.
		More specifically, we show that our assumption entails that for all $r$ starting with $r=\bar{R}$ and
		down to $r=1$, it holds that
		$p^{(\nicefrac{T}{2})}_r \geq \hat{p}^{(\nicefrac{T}{2})}_r$.
		
		Before we begin proving the induction, we choose a value for $\Omega$ that upholds
		$\Omega>(\nicefrac{T}{2})^{2}$ such that the following condition holds: for any $r \in [\bar{R}]$,
		the corresponding weight for the color $r$, i.e. $\Omega^r$, is strictly
		greater than $\Omega^{r-1} (\nicefrac{T}{2})^{2}$. Thus, adding the
		reward of even a single ball of color $r$ is always preferable over
		any possible amount of balls of color $r' < r$.
		
		We begin with the base case of $r=\bar{R}$. If $\hat{p}^{(\nicefrac{T}{2})} = 0$
		the claim is trivially satisfied. Otherwise, we assume by contradiction that
		$p^{(\nicefrac{T}{2})}_{\bar{R}} < \hat{p}^{(\nicefrac{T}{2})}_{\bar{R}}$.
		If $p^{(\nicefrac{T}{2})}_{\bar{R}} = 0$, then the weight of the color
		$\bar{R}$ is not part of the total reward $\rho$, and per our choice of $\Omega$
		it must hold that $\rho < \rho^*$ since $\rho^*$ does include a term of $\Omega^{\bar{R}}$ by definition. Now, we examine the last state of the
		trajectory $\p^{(1)}$, where there is a single ball left in the bucket.
		Per our choice of $\Omega$, if $p^{(1)}_{\bar{R}} = 0$, then once again
		$\rho < \rho^*$, implying that $p^{(1)}_{\bar{R}} = 1$. Following the
		same logic, for $j \in [p^{(\nicefrac{T}{2})}_{\bar{R}}]$, it holds that
		$p^{(j)}_{\bar{R}} = j$. Thus the total contribution of the $\bar{R}$'th
		weight is at most:
		\begin{equation}\label{eq:combinatoric_induction:base}
		\Omega^{\bar{R}} \left((\hat{p}^{(\nicefrac{T}{2})}_{\bar{R}} - p^{(\nicefrac{T}{2})}_{\bar{R}})
		\cdot p^{(\nicefrac{T}{2})}_{\bar{R}}
		+ \sum_{i=1}^{p^{(\nicefrac{T}{2})}_{\bar{R}}} i\right).
		\end{equation}
		This is because before spending
		all of the $p^{(\nicefrac{T}{2})}_{\bar{R}}$ balls of color $\bar{R}$ at the end,
		there are another $(\hat{p}^{(\nicefrac{T}{2})}_{\bar{R}} - p^{(\nicefrac{T}{2})}_{\bar{R}})$
		time-steps at which we add to the reward a value of $p^{(\nicefrac{T}{2})}_{\bar{R}}$.
		However, since Equation~\eqref{eq:combinatoric_induction:base} is strictly less than the corresponding contribution of $\Omega^{\bar{R}}$ in $\rho^*$:
		$\Omega^{\bar{R}} \sum_{i=1}^{\hat{p}^{(\nicefrac{T}{2})}} i$, then it follows that
		$\rho < \rho^*$, in contradiction to our assumption, which implies that to uphold the assumption the following must hold:
		$p^{(\nicefrac{T}{2})}_{\bar{R}} \geq \hat{p}^{(\nicefrac{T}{2})}_{\bar{R}}$, proving the induction base.
		
		Assuming our induction hypothesis holds for all $r' > r$, we show it also holds
		for $r$. Similar to our base case, if $\hat{p}^{(\nicefrac{T}{2})}_r = 0$ then
		our claim is trivially satisfied, and likewise if $p^{(\nicefrac{T}{2})}_r =0$,
		hence it remains to show that the case of
		$p^{(\nicefrac{T}{2})}_{\bar{R}} < \hat{p}^{(\nicefrac{T}{2})}_{\bar{R}}$ is
		not possible. First, according to our hypothesis, $\forall r' > r,
		p^{(\nicefrac{T}{2})}_{r'} \geq \hat{p}^{(\nicefrac{T}{2})}_{r'}$, and per our choice
		of $\Omega$, the contributions to the reward of all of the weights for $r' > r$,
		are at most $\sum_{r'=r+1}^{\bar{R}} \Omega^{r'}
		\sum_{i=1}^{\hat{p}^{(\nicefrac{T}{2})}_{r'}} i$, which is exactly equal to the
		corresponding contributions in $\rho^*$. This means that per our choice of
		$\Omega$ it suffices to show that the contributions originating in the color $r$
		are strictly less than the ones in $\rho^*$ to prove our hypothesis.
		In this optimal setting, the state of the bucket at time-step
		$j = \nicefrac{T}{2} - \sum_{r'=r+1}^{\bar{R}} \hat{p}^{(\nicefrac{T}{2})}_{r'}$
		must upholds $p^{(j)}_{r'} = \hat{p}^{(\nicefrac{T}{2})}_{r'}$ for $r' > r$, and
		zero otherwise. At this point, employing exactly the same logic as in
		our base case, the total contribution to the reward of the weight for the
		$r$'th color is at most:
		\begin{equation}\label{eq:combinatoric_induction:induction}
		\Omega^{r} \left((\hat{p}^{(\nicefrac{T}{2})}_{r} - p^{(\nicefrac{T}{2})}_{r})
		\cdot p^{(\nicefrac{T}{2})}_{r}
		+ \sum_{i=1}^{p^{(\nicefrac{T}{2})}_{r}} i\right),
		\end{equation}
		which is strictly less than the respective contribution in $\rho^*$.
		
	\end{proof}

	\section{Matricization Definition}\label{app:mat}
	
	Suppose $\A\in\R^{M{\times\cdots\times}M}$
	is a tensor of order $T$, and let $(I,J)$ be a partition of $[T]$, \ie~$I$
	and~$J$ are disjoint subsets of $[T]$ whose union gives~$[T]$.
	The \emph{matricization of $\A$ \wrt~the partition $(I,J)$}, denoted
	$\mat{\A}_{I,J}$, is the $M^{\abs{I}}$-by-$M^{\abs{J}}$ matrix holding the entries of $\A$ such that $\A_{d_1{\ldots}d_T}$ is placed in row index $1+\sum_{t=1}^{\abs{I}}(d_{i_t}-1)M^{\abs{I}-t}$ and column index $1+\sum_{t=1}^{\abs{J}}(d_{j_t}-1)M^{\abs{J}-t}$.
	
\end{document}